\documentclass[11pt]{article}
\usepackage[english]{babel}
\usepackage[left=1in,right=1in,top=1in,bottom=1.5in]{geometry} 
\usepackage[skip=5pt, indent=10pt]{parskip}
\usepackage{amssymb, amsmath, amsthm}
\usepackage{hyperref}
\usepackage{xcolor}
\usepackage{nicefrac}
\usepackage[margin=1cm]{caption}
\usepackage{natbib}
\usepackage{bbm}
\usepackage[scr]{rsfso}
\usepackage{comment}
\usepackage{algorithm}
\usepackage{algpseudocode}
\usepackage{times}
\usepackage{graphicx}
\usepackage{stackrel}
\usepackage{cases}

\definecolor{darkred}{HTML}{880000}
\definecolor{darkblue}{HTML}{000088}

\hypersetup{
  colorlinks,
  linkcolor={darkred},
  citecolor={darkblue}
}

\newcommand{\myendproof}{$\hfill\square$}

\renewcommand{\leq}{\leqslant}

\renewcommand{\geq}{\geqslant}

\renewcommand{\tilde}{\widetilde}

\newcommand{\eps}{\varepsilon}

\newcommand{\dd}{{\mathrm{d}}}

\newcommand{\NN}{\mathsf{NN}}

\newcommand{\N}{\mathbb N}

\newcommand{\R}{\mathbb R}
\newcommand{\Z}{\mathbb Z}

\newcommand{\bk}{{\mathbf k}}
\newcommand{\bl}{{\mathbf l}}
\newcommand{\blambda}{{\boldsymbol \lambda}}
\newcommand{\bnu}{{\boldsymbol \nu}}

\newcommand{\cH}{{\mathcal{H}}}

\newcommand{\cO}{{\mathcal{O}}}

\newcommand{\ceil}[1]{\lceil #1 \rceil}

\def\esssup{\operatornamewithlimits{esssup}}

\newcommand{\gelu}{\mathrm{GELU}}
\newcommand{\id}{\mathrm{id}}

\newcommand{\integral}[1]{\int\limits_{#1}}

\newtheorem{Th}{Theorem}[section]
\newtheorem{Lem}[Th]{Lemma}
\newtheorem{Def}[Th]{Definition}

\newtheorem{Co}[Th]{Corollary}

\title{Approximation Capabilities of Feedforward Neural Networks with GELU Activations}

\author{
Konstantin Yakovlev\thanks{HSE University, Russian Federation, kdyakovlev@hse.ru}
\and
Nikita Puchkin\thanks{HSE University, Russian Federation, npuchkin@hse.ru}
}

\date{}

\begin{document}

\maketitle

\begin{abstract}
    We derive an approximation error bound that holds simultaneously for a function and all its derivatives up to any prescribed order.
    The bounds apply to elementary functions, including multivariate polynomials, the exponential function, and the reciprocal function, and are obtained using feedforward neural networks with the Gaussian Error Linear Unit (GELU) activation.
    In addition, we report the network size, weight magnitudes, and behavior at infinity.
    Our analysis begins with a constructive approximation of multiplication, where we prove the simultaneous validity of error bounds over domains of increasing size for a given approximator.
    Leveraging this result, we obtain approximation guarantees for division and the exponential function, ensuring that all higher‑order derivatives of the resulting approximators remain globally bounded.
\end{abstract}

\tableofcontents

\section{Introduction}

We investigate simultaneous approximation of multivariate functions and their higher-order derivatives using deep feedforward neural networks.
Since approximating derivatives necessitates a smooth activation function, we employ networks with infinitely differentiable the Gaussian Error Linear Unit (GELU) \citep{hendrycks2016gaussian}.
Our choice is motivated by the fact that higher‑order derivatives of the GELU activation can be expressed in terms of Hermite polynomials, which enables simple and tractable bounds on their absolute values.
In addition, this activation function is widely adopted in state-of-the-art large language models \citep{devlin2019,raffel2020exploring,shoeybi2019megatron}, which further underscores its practical relevance.

The core of our constructive approach is the localized approximation of polynomials.
While prior works have developed approximation theory for smooth functions and their derivatives on fixed compact sets \citep{yarotskiy2017,de2021approximation,guhring2021approximation,belomestny2023simultaneous}, they do not provide error bounds beyond the original domain nor offer simultaneous guarantees across a sequence of increasingly large domains.
Despite recent advances in approximation of functions with noncompact domain presented in \cite{schwab2021deep,nuland2024noncompact,abdeljawad2024weighted}, the results either do not focus on the simultaneous approximation of derivatives or impose strong assumptions on the weight function of the underlying weighted $L^p$ space \citep{abdeljawad2024weighted}.
We bridge this gap by providing explicit control on how approximation errors for fundamental operations (like multiplication) scale as the domain size grows.
This allows us to construct an approximation of monomials with globally bounded higher‑order derivatives.
This properties are relevant in approximation of functions with unbounded domains including generative modelling \citep{oko2023diffusion,tang2024adaptivity,azangulov2024convergence,yakovlev2025generalization,fukumizu2025flow} and physics-informed neural networks \citep{abdo2024error,alejo2024error}.

Our key technical innovations are twofold.
First, we systematically employ a clipping operation on the network input.
By clipping the neural network input, we ensure that the derivatives are globally bounded, since they are bounded on a compact domain.
Second, we establish approximation guarantees for partition-of-unity functions in Sobolev seminorms, constructing functions with globally bounded derivatives and light tails.

As a consequence of these results, we provide approximation error bounds for the exponential function approximation and division approximation together with its derivatives, ensuring that their higher-order derivatives are globally bounded.
Consequently, we extend the approximation results for elementary functions presented in \cite{oko2023diffusion,yakovlev2025generalization} to Sobolev seminorms on domains of increasing size.

\medskip
\noindent
\textbf{Paper structure.}\quad
The remainder of the paper is structured as follows.
Section \ref{sec:prelim_not} establishes necessary preliminaries and notations.
In Section \ref{sec:main_res}, we present our main result on approximation error bounds.
Proofs not included in the main text are provided in the Appendix.

\medskip
\noindent
\textbf{Notation.}\quad
The set of non-negative integers is denoted by $\Z_+ = \{0, 1, 2, \dots\}$.
A multi-index $\bk \in \Z_+^d$, where $d \in \N$, is denoted in bold.
We also define $|\bk| = k_1 + k_2 + \ldots + k_d$, $\bk! = k_1! \cdot k_2! \cdot \ldots \cdot k_d!$
For a vector $v \in \R^d$ we define $v^\bk = v_1^{k_1} v_2^{k_2} \ldots v_d^{k_d}$.
For a function $f$ of $d$ variables, its weak derivative with respect to the multi-index $\bk \in \Z_+^d$ is denoted as
\begin{align*}
    \partial^\bk f = \frac{\partial^{|\bk|} f}{\partial x_1^{k_1} \partial x_2^{k_2} \ldots \partial x_d^{k_d}} .
\end{align*}
Throughout the paper, we employ the notation $f \lesssim g$ to indicate that $f = \cO(g)$.
If $f \lesssim g$ and $g \lesssim f$, then we write $f \asymp g$.
We frequently replace the expression for $\min\{a, b\}$ and $\max\{a, b\}$ with $a \vee b$ and $a \wedge b$, respectively.
For any $x > 0$, we define $\log(x) = \ln(x \vee e)$.

\section{Preliminaries and notations}
\label{sec:prelim_not}

\noindent
\textbf{Norms.}\quad
We denote the Euclidean norm of a vector $v$ as $\|v\|$, the maximal absolute value of its entries as $\|v\|\infty$, and the number of its non-zero entries as $\|v\|_0$.
Similarly, $\|A\|\infty$ and $\|A\|_0$ represent the maximal absolute value of entries and the number of non-zero entries of matrix $A$, respectively.
Finally, for a set $\Omega \subseteq \R^r$ and a function $f : \Omega \to \R^d$, we define
\begin{align*}
    \|f\|_{L^\infty(\Omega)} = \esssup_{x \in \Omega}\|f(x)\| .
\end{align*}

\medskip
\noindent
\textbf{Smoothness spaces.}\quad
We introduce the Sobolev space to characterize the regularity of functions in our analysis, and its definition is provided below.

\begin{Def}[Sobolev space]
    Let $\Omega \subseteq \R^r$ be an open set, and let $k \in \Z_+$.
    Then, the Sobolev space $W^{k, \infty}(\Omega)$ is defined as follows:
    \begin{align*}
        W^{k, \infty}(\Omega) = \{f \in L^\infty(\Omega) : \partial^\bk f \in L^\infty(\Omega) \quad \text{for every } \bk \in \Z_+^r \text{ with } |\bk| \leq k \} .
    \end{align*}
    Here, $L^\infty(\Omega)$ is the Lebesgue space.
    We define the Sobolev seminorm on $W^{k, \infty}(\Omega)$ as
    \begin{align*}
        |f|_{W^{k, \infty}(\Omega)} = \max_{\substack{\bk \in \Z_+^r, \; |\bk| = k}}\|\partial^\bk f\|_{L^\infty(\Omega)}.
    \end{align*}
    Finally, we define the Sobolev norm on $W^{k, \infty}(\Omega)$ as
    \begin{align*}
        \|f\|_{W^{k, \infty}(\Omega)} = \max_{0 \leq m \leq k}|f|_{W^{m, \infty}(\Omega)} .
    \end{align*}
\end{Def}

\medskip
\noindent
\textbf{Neural networks.}\quad
In this paper, we focus on feed-forward neural networks employing the Gaussian Error Linear Unit (GELU) activation function:
\begin{align*}
    \gelu(x) = x \cdot \Phi(x), \quad \Phi(x) = \frac{1}{\sqrt{2\pi}}\integral{-\infty}^x e^{-t^2 / 2} \dd t .
\end{align*}
The choice of GeLU is motivated by its infinite smoothness and bounded derivatives (see Lemma \ref{lem:gelu_seminorms_bound}).
For a vector $b = (b_1, \dots, b_r) \in \R^r$, we define the shifted activation function $\gelu_b : \R^r \to \R^r$ as
\begin{align*}
    \gelu_b(x) = (\gelu(x_1 - b_1), \dots, \gelu(x_r - b_r)), \quad x = (x_1, \dots, x_r) \in \R^r .
\end{align*}
Given a network depth $L \in \N$ and a vector of layer sizes $W = (W_0, W_1, \dots, W_L) \in \N^{L + 1}$.
Then, a neural network of depth $L$ and architecture $W$ is a function $f : \R^{W_0} \to \R^{W_L}$ such that
\begin{align}
    \label{eq:feed_forward_nn_def}
    f(x) = -b_L + A_L \circ \gelu_{b_{L - 1}} \circ A_{L - 1} \circ \gelu_{b_{L - 2}} \circ \dots \circ A_2 \circ \gelu_{b_1} \circ A_1 \circ x ,
\end{align}
where $A_j \in \R^{W_{j} \times W_{j - 1}}$ is a weight matrix and $b_j \in \R^{W_j}$ is a bias vector for all $j \in {1, \dots, L}$.
The maximum number of neurons of each layer is given by $\|W\|_\infty$ and is reffered to as the width of the neural network.
We define the class $\NN(L, W, S, B)$ of neural networks of the form \eqref{eq:feed_forward_nn_def} with at most $S$ non-zero weights and the weight magnitude $B$ as follows:
\begin{align*}
    \NN(L, W, S, B) = \left\{f \text{ of the form } \eqref{eq:feed_forward_nn_def} \, : \, \sum_{j = 1}^L (\|A_j\|_0 + \|b_j\|_0) \leq S, \; \max_{1 \leq j \leq L} \|A_j\|_\infty \vee \|b_j\|_\infty \leq B \right\} .
\end{align*}

\section{Main results}
\label{sec:main_res}

This section presents our main results.
Specifically, Subsection \ref{subsec:elem_oper} details approximation error bounds for elementary operations, including the identity function, partition of unity, and square operation.
Subsection \ref{subsec:monom_approx} elaborates on the approximation of monomials.
Finally, Subsection \ref{subsec:exp_div_approx} provides approximation error bounds for the exponentiation and division operations.

\subsection{Approximation of elementary operations}
\label{subsec:elem_oper}

Passing the output of one layer to a non-adjacent layer is frequently beneficial.
Note that ReLU activation allows for an exact identity mapping \citep{nakada20}.
However, in the case of GELU, an approximate mapping is guaranteed, as demonstrated by the following lemma, which provides an approximation error bound for a single-layer neural network.

\begin{Lem}[approximation of identity operation]
    \label{lem:id_gelu_approx}
    Let $m \in \N$ and let $\id(x) = x$.
    Then, for any $\eps \in (0, 1)$ there exists $\varphi_{id} \in \NN(L, W, S, B)$ satisfying
    \begin{align*}
        \|\varphi_{id} - \id\|_{W^{m, \infty}([-C, C])} \leq C^2 \eps, \quad \text{for all } C \geq 1 .
    \end{align*}
    Furthermore, $L = 2$, $\|W\|_\infty = 1$, $S = 3$, and $\log B \lesssim \log(1 / \eps) + \log m$.
\end{Lem}

\begin{proof}
    The proof follows the same approach as outlined in \cite{scarselli98}.
    We let
    \begin{align*}
        \varphi_{id}(x) = -\frac{R \cdot \gelu(0)}{\partial^1\gelu(0)} + \frac{R}{\partial^1\gelu(0)}\gelu\left(\frac{x}{R}\right),
    \end{align*}
    where $R > 0$ and will be determined later in the proof.
    We also emphasize that the form of $\varphi_{id}$ is valid, since $\partial^1\gelu(0) = 1/2$.
    Taylor expansion suggests that for any $x \in [-C, C]$ it holds that
    \begin{align*}
        |\varphi_{id}(x) - x| \leq \frac{|\gelu|_{W^{2, \infty}(\R)}C^2}{2 \cdot \partial^1\gelu(0) R}.
    \end{align*}
    Similarly, we deduce that
    \begin{align*}
        |\partial^1\varphi_{id}(x) - 1| \leq \frac{|\gelu|_{W^{2, \infty}(\R)} C}{\partial^1\gelu(0) R}.
    \end{align*}
    Additionally, for any $k \geq 2$ we find that
    \begin{align*}
        |\varphi_{id} - \id|_{W^{k, \infty}(\R)}
        = |\varphi_{id}|_{W^{k, \infty}(\R)}
        \leq \frac{|\gelu|_{W^{k, \infty}(\R)}}{\partial^1\gelu(0) R^{k - 1}}.
    \end{align*}
    Therefore, choosing
    \begin{align*}
        R = \max_{2 \leq k \leq m}\left( \frac{|\gelu|_{W^{k, \infty}(\R)}}{\partial^1\gelu(0)\eps} \right)^{1 / (k - 1)} \vee 1,
    \end{align*}
    ensures that for any $C \geq 1$
    \begin{align*}
        \max_{0 \leq k \leq 1}|\varphi_{id} - \id|_{W^{k, \infty}([-C, C])} \leq C^2 \eps,
        \quad \max_{2 \leq k \leq m} |\varphi_{id} - \id|_{W^{k, \infty}(\R)} \leq \eps .
    \end{align*}
    Therefore, it holds that
    \begin{align*}
        \|\varphi_{id} - \id\|_{W^{m, \infty}([-C, C])} \leq C^2 \eps, \quad \text{for all } C \geq 1 .
    \end{align*}
    We next specify the configuration of $\varphi_{id}$.
    Clearly, $L = 2$, $\|W\|_\infty = 1$ and $S = 3$.
    As for the weight magnitude, we apply Lemma \ref{lem:gelu_seminorms_bound}, arriving at
    \begin{align*}
        \log B
        \lesssim \log(1 / \eps) + \max_{2 \leq k \leq m}\frac{\log\left(|\gelu|_{W^{k, \infty}(\R)} \vee 1\right)}{k-1}
        \lesssim \log(1 / \eps) + \max_{2 \leq k \leq m}\frac{\log\left((k + 1)!\right)}{k - 1}.
    \end{align*}
    Now Stirling's approximation implies that $\log ((k + 1)!) \lesssim k \log k$, and thus, 
    \begin{align*}
        \log B \lesssim \log(1 / \eps) + \log m.
    \end{align*}
    The proof is complete.
    
\end{proof}

The following lemma generalizes the result presented in Lemma \ref{lem:id_gelu_approx} to the case of multiple layers.

\begin{Lem}[approximation of identity operation with multiple layers]
    \label{lem:id_deep_gelu_approx}
    Let $m \in \N$ and let $\id : x \mapsto x$.
    Then, for every $\eps \in (0, 1)$, every $L \in \N$ with $L \geq 2$, and every $K \geq 1$, there exists $\varphi_{id} \in \NN(L, W, S, B)$ such that
    \begin{align*}
        (i) &\quad \|\varphi_{id} - \id\|_{W^{m, \infty}([-K, K])} \leq \eps, \\
        (ii) &\quad \|\varphi_{id}\|_{W^{m, \infty}(\R)} \leq \exp\{\cO(m \log(m\log(1 / \eps)) + \log(2K))\}.
    \end{align*}
    Moreover, it holds that
    \begin{align*}
        \|W\|_\infty \lesssim 1, \quad S \lesssim L,
        \quad \log B \lesssim (m + L)\log m + \log(1 / \eps) + m\log(K).
    \end{align*}
\end{Lem}

The proof of Lemma \ref{lem:id_deep_gelu_approx} is moved to Appendix \ref{sec:lem_id_deep_gelu_approx_proof}.
Next, we move to the approximation of partition of unity, a crucial component in the framework of localized Taylor polynomials \citep{guhring2021approximation,de2021approximation}.
First, we approximate the Heaviside step function, as presented in the following lemma.

\begin{Lem}[Approximation of Heaviside step function]
    \label{lem:heaviside_gelu_approx}
    For every $\eps \in (0, 1)$, every $\varkappa \in (0, 1)$, and every $m \in \N$, there exists a GELU network $\varphi_\varkappa \in \NN(L, W, S, B)$ such that
    \begin{align*}
        (i) &\quad \|\varphi_\varkappa\|_{W^{m, \infty}(\R)} \leq \exp\left\{\cO(m\log(m \log(1 / \eps) / \varkappa))\right\}, \\
        (ii) &\quad \|\varphi_\varkappa\|_{W^{m, \infty}((-\infty, -\varkappa])} \vee \|1 - \varphi_\varkappa\|_{W^{m, \infty}([\varkappa, +\infty))} \leq \eps
    \end{align*}
    Moreover, $\varphi_\varkappa$ has $L = 2$, $\|W\|_\infty \vee S \lesssim 1$, and $\log B \lesssim m \log(m / \varkappa) + \log (1 / \eps)$.
\end{Lem}

\begin{proof}
    Let
    \begin{align*}
        \eta(x) = \frac{\gelu(x + \eps_0) - \gelu(x - \eps_0)}{2\eps_0}, \quad x \in \R,
    \end{align*}
    where $\eps_0 \in (0, 1)$ will be determined later.
    Note that $\eta$ approximates $\partial^1\gelu$ is Sobolev norm.
    Formally, the Taylor expansion suggests that for any $k \in \Z_+$ and $x \in \R$ we have 
    \begin{align*}
        |\partial^k\eta(x) - \partial^{k + 1}\gelu(x)| 
        &= \frac{|\partial^k\gelu(x + \eps_0) - \partial^k\gelu(x - \eps_0) - 2\eps \partial^{k + 1}\gelu(x)|}{2\eps_0} \\
        &\leq \frac{\eps_0^2}{6}|\gelu|_{W^{k + 3, \infty}(\R)}.
    \end{align*}
    Hence, it holds that
    \begin{align}
        \label{eq:eta_gelu_1_approx}
        \|\eta - \partial^1\gelu\|_{W^{m, \infty}(\R)} \leq \frac{\eps_0^2}{6}\max_{0 \leq k \leq m} |\gelu|_{W^{k + 3, \infty}(\R)}.
    \end{align}
    Now let $\varphi_\varkappa(x) = \eta(\alpha x)$ for $\alpha \geq 1$ that will be optimized later.
    Therefore, triangle inequality yields
    \begin{align*}
        \|\varphi_\varkappa\|_{W^{m, \infty}([-\infty, -\varkappa])}
        &\leq \alpha^m \|\eta\|_{W^{m, \infty}((-\infty, -\alpha\varkappa])} \\
        &\leq \alpha^m \left(\|\partial^1\gelu - \eta\|_{W^{m, \infty}((-\infty, -\alpha\varkappa])} + \|\partial^1\gelu\|_{W^{m, \infty}((-\infty, -\alpha\varkappa])} \right).
    \end{align*}
    Thus, Lemma \ref{lem:gelu_seminorms_bound} together with \eqref{eq:eta_gelu_1_approx} implies that
    \begin{align*}
        \|\varphi_\varkappa\|_{W^{m, \infty}([-\infty, -\varkappa])}
        \leq \alpha^m\left(\frac{\eps_0^2}{6}(m + 4)\sqrt{(m + 1)!} + 2e^{-\alpha^2\varkappa^2 / 4}\sqrt{(m + 1)!}\right)
    \end{align*}
    Setting $\alpha = 2\varkappa^{-1}\sqrt{2\log(1 / \eps_0)} \geq 1$ ensures that
    \begin{align*}
        \|\varphi_\varkappa\|_{W^{m, \infty}([-\infty, -\varkappa])}
        \leq \alpha^m \eps_0^2 (m + 3) \sqrt{(m + 1)!}
        \leq (8\varkappa^{-2}\log(1 / \eps_0))^{m / 2}\eps_0^2 (m + 3)\sqrt{(m + 1)!}.
    \end{align*}
    Using the fact that
    \begin{align*}
        \sup_{\eps \in (0, 1)}\eps (\log(1 / \eps))^{m / 2} \leq \left(\frac{m}{2 e}\right)^{m / 2} \leq m^{m / 2}
    \end{align*}
    we find that
    \begin{align*}
        \|\varphi_\varkappa\|_{W^{m, \infty}((-\infty, -\varkappa])}
        \leq (8\kappa^{-2}m)^{m / 2}\eps_0 (m + 3)\sqrt{(m + 1)!}.
    \end{align*}
    Hence, setting
    \begin{align}
        \label{eq:phi_heav_eps_0_def}
        \eps_0 = \left((8\kappa^{-2}m)^{m / 2} (m + 3)\sqrt{(m + 1)!}\right)^{-1} \eps \in (0, 1)
    \end{align}
    ensures that
    \begin{align*}
        \|\varphi_\varkappa\|_{W^{m, \infty}((-\infty, -\varkappa])}
        \leq \eps.
    \end{align*}
    Using similar argument, we also deduce that
    \begin{align*}
        \|1 - \varphi_\varkappa\|_{W^{m, \infty}([\kappa, +\infty))}
        \leq \eps.
    \end{align*}
    We next note that in view of \eqref{eq:eta_gelu_1_approx} and Lemma \ref{lem:gelu_seminorms_bound}, it holds that
    \begin{align*}
        \|\varphi_\varkappa\|_{W^{m, \infty}(\R)}
        &\leq \alpha^m (\|\eta - \partial^1\gelu\|_{W^{m, \infty}(\R)} + \|\partial^1\gelu\|_{W^{m, \infty}(\R)}) \\
        &\leq \alpha^m (\frac{\eps_0^2}{6}(m + 4)\sqrt{(m + 1)!} + (m + 2)\sqrt{(m - 1)!}).
    \end{align*}
    The choice of $\alpha$ indicates that
    \begin{align*}
        \|\varphi_\varkappa\|_{W^{m, \infty}(\R)}
        \leq 2 \alpha^m (m + 2)\sqrt{(m + 1)!}
        \leq 2 (8\kappa^{-2}\log(1 / \eps_0))^{m / 2} (m + 2)\sqrt{(m + 1)!}.
    \end{align*}
    Now the choice of $\eps_0$ from \eqref{eq:phi_heav_eps_0_def} yields
    \begin{align*}
        \|\varphi_\varkappa\|_{W^{m, \infty}(\R)}
        \leq \exp\left\{\cO(m\log(m \log(1 / \eps) / \varkappa))\right\}.
    \end{align*}
    Finally, we specify the configuration of $\varphi_\varkappa$.
    Clearly, $L = 2$, $\|W\|_\infty \vee S \lesssim 1$, and
    \begin{align*}
        \log B
        \lesssim \log(\alpha) + \log(1 / \eps_0)
        \lesssim m \log(m / \varkappa) + \log (1 / \eps).
    \end{align*}
    The proof is complete.
    
\end{proof}

Subsequently, we use Lemma \ref{lem:heaviside_gelu_approx} to approximate a partition of unity.
Following \cite{yakovlev2025generalization}, we use non-uniform partition, a key element in approximating the division operation.
The result is presented below.

\begin{Lem}[partition of unity approximation]
    \label{lem:pou_gelu_approx}
    Define $a_i = 2^{-N + i}$ for each $i \in \{0, 1, \dots, N\}$, where $N \in \N$ and $N \geq 3$.
    Then, for every $\eps \in (0, 1)$ and every $m \in \N$, there exist $\{\psi_i\}_{i=1}^N$, with $\psi_i \in \NN(L, W, S, B)$ for each $1 \leq i \leq N$, such that
    \begin{align*}
        (i) &\quad \sum_{i=1}^N \psi_i(x) = 1, \quad \text{for all } x \in \R, \\
        (ii) &\quad \max_{1 \leq i \leq N} \|\psi_i\|_{W^{m, \infty}(\R)} \leq \exp\{\cO(mN + m\log(m\log(1 / \eps)))\}, \\
        (iii) &\quad \|\psi_N\|_{W^{m, \infty}(-\infty, a_{N - 2}]} \vee \|\psi_1\|_{W^{m, \infty}([a_2, +\infty))} \vee  \max_{2 \leq i \leq N - 1}\|\psi_i\|_{W^{m, \infty}(\R \setminus (a_{i - 2}, a_{i + 1}))}
        \leq \eps,
    \end{align*}
    Furthermore, $L = 2$, $\|W\|_\infty \vee S \lesssim 1$ and $\log B \lesssim \log(1 / \eps) + mN + m\log m$.
\end{Lem}

\begin{proof}
    Next, for a Heaviside function approximation $\varphi_{a_0}$ from Lemma \ref{lem:heaviside_gelu_approx} formulated with accuracy parameter $\eps / 2$ and $\kappa = a_0$, we define
    \begin{align*}
        \psi_i(x) =
        \begin{cases}
            1 - \varphi_{a_0}(x - a_1), \quad &i = 1, \\
            \varphi_{a_0}(x - a_{i - 1}) - \varphi_{a_0}(x - a_i), \quad &i \in \{2, \dots, N - 1\}, \\
            \varphi_{a_0}(x - a_{N - 1}), \quad &i = N
        \end{cases}
    \end{align*}
    It is clear that for all $x \in \R$
    \begin{align*}
        \sum_{i=1}^N \psi_i(x) = 1.
    \end{align*}
    In other words, $\{\psi_i\}_{i=1}^N$ forms a partition of unity.
    Now derive the behavior of tails for each $\psi_i$.
    First, note that for each $1 \leq i \leq N$ we have that
    \begin{align}
        \label{eq:heavi_left_tail}
        \|\varphi_{a_0}(\cdot - a_i)\|_{W^{m, \infty}((-\infty, a_{i - 1}])}
        \leq \|\varphi_{a_0}\|_{W^{m, \infty}((-\infty, -a_0])} \leq \eps / 2
    \end{align}
    and similarly
    \begin{align}
        \label{eq:heavi_right_tail}
        \|1 - \varphi_{a_0}(\cdot - a_i)\|_{W^{m, \infty}([a_{i + 1}, +\infty))}
        \leq \|1 - \varphi_{a_0}\|_{W^{m, \infty}([a_0, +\infty))}
        \leq \eps / 2.
    \end{align}
    Therefore,
    \begin{align*}
        \|\psi_N\|_{W^{m, \infty}(-\infty, a_{N - 2}]} \vee \|\psi_1\|_{W^{m, \infty}([a_2, +\infty))} \leq \eps.
    \end{align*}
    Next, for any $2 \leq i \leq N - 1$ it holds that
    \begin{align*}
        \|\psi_i\|_{W^{m, \infty}(\R \setminus (a_{i - 2}, a_{i + 1}))}
        = \|\psi_i\|_{W^{m, \infty}((-\infty, a_{i - 2}])} \vee \|\psi_i\|_{W^{m, \infty}([a_{i + 1}, +\infty))}
    \end{align*}
    First, from \eqref{eq:heavi_left_tail} we find that
    \begin{align*}
        \|\psi_i\|_{W^{m, \infty}((-\infty, a_{i - 2}])}
        \leq \|\varphi_{a_0}(\cdot - a_{i - 1})\|_{W^{m, \infty}((-\infty, a_{i - 2}])} + \|\varphi_{a_0}(\cdot - a_i )\|_{W^{m, \infty}((-\infty, a_{i - 2}])}
        \leq \eps.
    \end{align*}
    Second, \eqref{eq:heavi_right_tail} implies that
    \begin{align*}
        \|\psi_i\|_{W^{m, \infty}([a_{i + 1}, +\infty))}
        \leq \|1 - \varphi_{a_0}(\cdot - a_{i - 1})\|_{W^{m, \infty}([a_{i + 1}, +\infty))} + \|1 - \varphi_{a_0}(\cdot - a_i)\|_{W^{m, \infty}([a_{i + 1}, +\infty))}
        \leq \eps.
    \end{align*}
    Thus, we arrive at 
    \begin{align*}
        \|\psi_N\|_{W^{m, \infty}(-\infty, a_{N - 2}]} \vee \|\psi_1\|_{W^{m, \infty}([a_2, +\infty))} \vee  \max_{2 \leq i \leq N - 1}\|\psi_i\|_{W^{m, \infty}(\R \setminus (a_{i - 2}, a_{i + 1}))}
        \leq \eps.
    \end{align*}
    Now we focus on the behavior of each $\psi_i$ on the real line.
    Formally, Lemma \ref{lem:heaviside_gelu_approx} suggests that for any $1 \leq i \leq N$ 
    \begin{align*}
        \|\psi_i\|_{W^{m, \infty}(\R)}
        \leq 2 \|\varphi_{a_0}\|_{W^{m, \infty}(\R)}
        \leq \exp \left\{\cO( m\log(m \log(1 / \eps) / a_0) ) \right\}.
    \end{align*}
    Recall that $a_0 = 2^{-N}$.
    Hence, it holds that
    \begin{align*}
        \|\psi_i\|_{W^{m, \infty}(\R)} \leq \exp\{\cO(mN + m\log(m\log(1 / \eps)))\}.
    \end{align*}
    We now specify the configuration for each $\psi_i$.
    Using the configuration of $\varphi_{a_0}$ outlined in Lemma \ref{lem:heaviside_gelu_approx} and parallelization argument from Lemma \ref{lem:paral_nn}, we conclude that
    \begin{align*}
        &L = 2, \quad \|W\|_\infty \vee S \lesssim 1, \\
        &\log B \lesssim \log(1 / \eps) + m\log(m / a_0)
        \lesssim \log(1 / \eps) + mN + m\log m.
    \end{align*}
    The proof is finished.
    
\end{proof}

Next, we aim to approximate the clipping operation, which is essential for controlling the Sobolev norm at infinity of the approximator.
The following lemma demonstrates the existence of a shallow GELU network for approximating clipping.

\begin{Lem}[approximation of clipping operation]
    \label{lem:clip_gelu_approx}
    For every $A \geq 1$, every $\eps \in (0, 1)$, and every $m \in \N$, there exists $\varphi_{clip} \in \NN(L, W, S, B)$ such that
    \begin{align*}
        (i) &\quad \|\varphi_{clip} - \id \|_{W^{m, \infty}([-A, A])} \leq \eps, \\
        (ii) &\quad \|\varphi_{clip} + A + 1/2\|_{W^{m, \infty}((-\infty, -A - 1])} \vee \|\varphi_{clip} - A - 1/2\|_{W^{m, \infty}([A + 1, +\infty))}
        \leq \eps, \\
        (iii) &\quad \|\varphi_{clip}\|_{W^{m, \infty}(\R)} \leq \exp\{\cO(m \log m + m\log\log(1 / \eps) + \log(2A) )\}, \\
        (iv) &\quad \|\varphi_{clip}\|_{W^{0, \infty}(\R)} \leq A + 5 / 2, \\
        (v) &\quad \|\varphi_{clip} + A + 1/2\|_{W^{0, \infty}((-\infty, -A])} \vee \|\varphi_{clip} - A - 1/2\|_{W^{0, \infty}([A, +\infty))} \leq \eps + 1, \\
        (vi) &\quad |\varphi_{clip}|_{W^{k, \infty}(\R)} \leq \exp\{\cO(k \log m + k \log\log(1 / \eps) )\} .
    \end{align*}
    Moreover, $\varphi_{clip}$ has $L = \|W\|_\infty = 2$, $S = 7$ and $\log B \lesssim \log(A m / \eps)$.
    
\end{Lem}

\begin{proof}
    Define
    \begin{align}
        \label{eq:phi_clip_def}
        \varphi_{clip}(x) = \alpha^{-1}\gelu(\alpha(x + A + 1/2)) - \alpha^{-1}\gelu(\alpha(x - A - 1/2)) - A - 1/2, \quad x \in \R,
    \end{align}
    where $\alpha \geq 1$ will be determined later in the proof.
    Therefore, Lemma \ref{lem:gelu_seminorms_bound} implies that
    \begin{align*}
        \|\varphi_{clip} - \id\|_{W^{m, \infty}([-A, A])}
        &\leq \alpha^m\|\gelu - \id\|_{W^{m, \infty}([\alpha / 2, +\infty))} + \alpha^m \|\gelu\|_{W^{m, \infty}((-\infty, -\alpha / 2])} \\
        &\leq 4\alpha^m \exp(-\alpha^2 / 16).
    \end{align*}
    We next note that
    \begin{align}
        \label{eq:alpha_m_exp_sup}
        \sup_{\alpha > 0}\alpha^m \exp(-\alpha^2 / 32)
        \leq \exp\{\cO(m\log m)\},
    \end{align}
    which implies that
    \begin{align*}
        \|\varphi_{clip} - \id\|_{W^{m, \infty}([-A, A])}
        &\leq \exp\{\cO(m\log m)\} \exp\{-\alpha^2 / 32\}.
    \end{align*}
    Therefore, setting
    \begin{align}
        \label{eq:alpha_clip}
        \alpha \asymp \sqrt{m\log m + \log(1 / \eps)}
    \end{align}
    guarantees that
    \begin{align}
        \label{eq:phi_clip_id_approx}
        \|\varphi_{clip} - \id\|_{W^{m, \infty}([-A, A])} \leq \eps.
    \end{align}
    We next focus on the behavior of tails of $\varphi_{clip}$
    \begin{align*}
        \|\varphi_{clip}\|_{W^{m, \infty}(\R \setminus (-A-1, A+1))}
        = \|\varphi_{clip}\|_{W^{m, \infty}((-\infty, -A-1])} \vee \|\varphi_{clip}\|_{W^{m, \infty}([A+1, +\infty))}.
    \end{align*}
    We also note that
    \begin{align*}
        \|\varphi_{clip} + A + 1/2\|_{W^{m, \infty}((-\infty, -A - 1])}
        \leq 2\alpha^m\|\gelu\|_{W^{m, \infty}((-\infty, -\alpha / 2])}
    \end{align*}
    and, similarly,
    \begin{align*}
        \|\varphi_{clip} - A - 1/2\|_{W^{m, \infty}([A + 1, +\infty))}
        \leq 2\alpha^m\|\gelu - \id\|_{W^{m, \infty}([\alpha / 2, +\infty))}.
    \end{align*}
    From Lemma \ref{lem:gelu_seminorms_bound} we find that
    \begin{align*}
        \|\varphi_{clip} + A + 1/2\|_{W^{m, \infty}((-\infty, -A - 1])} \vee \|\varphi_{clip} - A - 1/2\|_{W^{m, \infty}([A + 1, +\infty))}
        \leq 4\alpha^m \exp(-\alpha^2 / 16)
        \leq \eps,
    \end{align*}
    where the last inequality uses \eqref{eq:alpha_m_exp_sup} and \eqref{eq:alpha_clip}.
    Therefore, due to the triangle inequality we have that
    \begin{align}
        \label{eq:phi_clip_tail}
        \|\varphi_{clip}\|_{W^{m, \infty}(\R \setminus (-A - 1, A + 1))}
        \leq \eps + A + 1 / 2.
    \end{align}
    We next derive the Sobolev norm of $\varphi_{clip}$ on the real line.
    Using \eqref{eq:phi_clip_tail}, we have that
    \begin{align*}
        \|\varphi_{clip}\|_{W^{m, \infty}(\R)}
        &= \|\varphi_{clip}\|_{W^{m, \infty}(\R \setminus (-A-1, A+1))} \vee \|\varphi_{clip}\|_{W^{m, \infty}([-A-1, A+1])} \\
        &\leq (\eps + A + 1/2) \vee \|\varphi_{clip}\|_{W^{m, \infty}([-A-1, A+1])}.
    \end{align*}
    Lemma \ref{lem:gelu_seminorms_bound} together with \eqref{eq:alpha_clip} implies that
    \begin{align*}
        &\|\varphi_{clip}\|_{W^{m, \infty}([-A-1, A+1])} \\
        &\quad \leq 2 (\alpha^m \|\partial^1\gelu\|_{W^{m - 1, \infty}(\R)} \vee \alpha^{-1}\|\gelu\|_{W^{0, \infty}([-2A - 3/2, 2A + 3/2])}) + A + 1 / 2 \\
        &\quad \leq \exp\{\cO(m \log m + m\log\log(1 / \eps) + \log(2A) )\}.
    \end{align*}
    Therefore,
    \begin{align*}
        \|\varphi_{clip}\|_{W^{m, \infty}(\R)}
        \leq \exp\{\cO(m \log m + m\log\log(1 / \eps) + \log(2A) )\} .
    \end{align*}
    Similarly, for any $1 \leq k \leq m$, Lemma \ref{lem:gelu_seminorms_bound} in conjunction with \eqref{eq:phi_clip_def} and \eqref{eq:alpha_clip} yields
    \begin{align*}
        |\varphi_{clip}|_{W^{k, \infty}(\R)}
        \leq 2 \alpha^{k - 1}|\gelu|_{W^{k, \infty}(\R)}
        \leq \exp\{\cO(k \log m + k \log\log(1 / \eps))\} .
    \end{align*}
    In addition, from Lemma \ref{lem:gelu_seminorms_bound} we find that
    \begin{align*}
        &\|\varphi_{clip} - A - 1/2\|_{W^{0, \infty}([A, A + 1])} \\
        &\quad \leq \|\gelu - \id\|_{W^{0, \infty}([\alpha(2A + 1/2), +\infty))} + \|\id - A - 1/2\|_{W^{0, \infty}([A, A + 1])} + \alpha^{-1}\|\gelu\|_{W^{0, \infty}([-\alpha/2, \alpha/2])} \\
        &\quad \leq \eps + 1.
    \end{align*}
    Similarly,
    \begin{align*}
        \|\varphi_{clip} + A + 1/2\|_{W^{0, \infty}([-A-1, -A])} \leq \eps + 1.
    \end{align*}
    Therefore, from \eqref{eq:phi_clip_id_approx} and \eqref{eq:phi_clip_tail} we deduce that
    \begin{align*}
        \|\varphi_{clip}\|_{W^{0, \infty}(\R)}
        \leq \eps + A + 3/2
        \leq A + 5/2.
    \end{align*}
    Finally, we specify the configuration of $\varphi_{clip}$.
    The choice of $\alpha$ from \eqref{eq:alpha_clip} in conjunction with \eqref{eq:phi_clip_def} suggest that
    \begin{align*}
        L = \|W\|_\infty = 2,
        \quad S = 7,
        \quad \log B \lesssim \log (\alpha \vee \alpha^{-1} \vee A) \lesssim \log(A m / \eps).
    \end{align*}
    The proof is complete.
    
\end{proof}

\subsection{Approximation of monomials}
\label{subsec:monom_approx}

Now, we focus on approximating polynomials.
A key starting point is the approximation of the square operation, as demonstrated in the following lemma.

\begin{Lem}[approximation of square operation]
\label{lem:square_approx}
    Define $f_{sq} : x \mapsto x^2$.
    Then, for every $\eps \in (0, 1)$ and every $m \in \N$, there exists $\varphi_{sq} \in \NN(L, W, S, B)$ such that
    \begin{align*}
        \|\varphi_{sq} - f_{sq}\|_{W^{m, \infty}([-C, C])} \leq C^3 \eps, \quad \text{for all } C \geq 1 .
    \end{align*}
    Furthermore, $L = \|W\|_\infty = 2$, $S = 6$ and $\log B \lesssim  \log(1 / \eps) + \log m$.
\end{Lem}

\begin{proof}
    Inspired by \cite[Theorem 2]{scarselli98}, we let
    \begin{align}
        \label{eq:phi_sq_def}
        \varphi_{sq}(x) = \frac{R^2}{\partial^2\gelu(0)}\left(\gelu\left(\frac{2x}{R}\right) - 2\cdot \gelu\left(\frac{x}{R}\right) \right),
    \end{align}
    where $R > 0$.
    We also highlight that $\partial^2\gelu(0) = \sqrt{2 / \pi}$.
    Using Taylor expansion it can be shown for any $x \in [-C, C]$ that
    \begin{align*}
        \left|\varphi_{sq}(x) - x^2\right|
        = \frac{|x|^3 \cdot |4\partial^3\gelu(\xi) - \partial^3\gelu(\zeta) |}{3 R  \cdot \partial^2\gelu(0)}
        \leq \frac{5 C^3|\gelu|_{W^{3, \infty}(\R)}}{3R \cdot \partial^2\gelu(0)},
    \end{align*}
    and similarly for the derivatives
    \begin{align*}
        \left|\partial^1\varphi_{sq}(x) - 2x\right|
        = \frac{2|x|^2 \cdot | 2\partial^3\gelu(\tilde\xi) - \partial^3\gelu(\tilde\zeta) |}{R \cdot |\partial^2\gelu(0)|}
        \leq \frac{6C^2|\gelu|_{W^{3, \infty}(\R)}}{R \cdot \partial^2\gelu(0)}
    \end{align*}
    and the second derivatives
    \begin{align*}
        |\partial^2\varphi_{sq}(x) - 2|
        = \frac{|x| \cdot |8\partial^3\gelu(\xi^\circ) - 2\partial^3\gelu(\zeta^\circ)|}{R \cdot \partial^2\gelu(0)}
        \leq \frac{10 C |\gelu|_{W^{3, \infty}(\R)}}{R \cdot \partial^2\gelu(0)},
    \end{align*}
    where $\xi, \tilde{\xi}, \xi^\circ$ and $\zeta, \tilde{\zeta}, \zeta^\circ$ are all within the interval defined by the origin and $x$.
    To proceed, we derive an explicit bound for the derivatives of order $k$ with $k \geq 3$ and $x \in \R$ as
    \begin{align*}
        |\partial^k\varphi_{sq}(x)|
        = \frac{1}{R^{k - 2}\partial^2\gelu(0)} \left| 2^k\partial^k\gelu\left(\frac{2x}{R}\right) - 2 \cdot \partial^k\gelu\left(\frac{x}{R}\right) \right|
        \leq \frac{(2^k + 2)|\gelu|_{W^{k, \infty}(\R)}}{R^{k - 2}\partial^2\gelu(0)}.
    \end{align*}
    Therefore, from Lemma \ref{lem:gelu_seminorms_bound} we find that
    \begin{align*}
        |\partial^k\varphi_{sq}(x)| \leq \frac{(k + 1)(2^k + 2)}{R^{k - 2}\partial^2\gelu(0)} \sqrt{\frac{(k - 2)!}{2\pi}}
        \leq \frac{2^{k + 2}k}{R^{k - 2} \partial^2\gelu(0)} \sqrt{\frac{(k - 2)!}{2\pi}}.
    \end{align*}
    Hence, setting
    \begin{align*}
        R
        = \frac{10 |\gelu|_{W^{3, \infty}(\R)}}{\partial^2\gelu(0) \cdot \eps}
        \vee \max_{3 \leq k \leq m}\left(\frac{2^{k + 2} k}{\partial^2\gelu(0)} \sqrt{\frac{(k - 2)!}{2\pi}}\right)^{1 / (k - 2)},
    \end{align*}
    we ensure that for any $C \geq 1$
    \begin{align*}
        \max_{0 \leq k \leq 2} |\varphi_{sq} - f_{sq}|_{W^{k, \infty}([-C, C])} \leq C^{3 - k}\eps,
        \quad \max_{3 \leq k \leq (m \vee 3)} |\varphi_{sq} - f_{sq}|_{W^{k, \infty}(\R)} \leq \eps.
    \end{align*}
    This observation yields
    \begin{align*}
        \|\varphi_{sq} - f_{sq}\|_{W^{m, \infty}([-C, C])} \leq C^3 \eps, \quad \text{for all } C \geq 1 .
    \end{align*}
    The definition of $\varphi_{sq}$ given in \eqref{eq:phi_sq_def} suggests that $L = \|W\|_\infty = 2$, $S = 6$ and
    \begin{align*}
        \log B \lesssim \log (R^2 \vee R^{-1} \vee 2) \lesssim \log(1 / \eps) + \log m,
    \end{align*}
    where the last inequality uses Stirling's approximation.
    The proof is finished.
    
\end{proof}

Comparing our result from Lemma \ref{lem:square_approx} to that presented in \cite[Proposition 4.7]{guhring2021approximation}, we observe that we provide approximation guarantees for the entire real line, rather than limiting our results to a specific segment.
This is a key advantage for approximating functions on unbounded domains.
Subsequently, we derive a straightforward corollary that provides an approximation error bound for the multiplication of two numbers.

\begin{Co}[approximation of two number multiplication]
    \label{co:multi_approx_gelu}
    Define $\mathrm{prod}_2 : (x, y) \mapsto x \cdot y$, and let $m \in \N$ be arbitrary.
    Then, for any $\eps \in (0, 1)$, there exists $\varphi_{mul} \in \NN(L, W, S, B)$ satisfying
    \begin{align*}
        \|\varphi_{mul} - \mathrm{prod}_2\|_{W^{m, \infty}([-C, C]^2)}
        \leq C^3 \eps ,
        \quad \text{for all } C \geq 1 .
    \end{align*}
    In addition, $L = 2$, $\|W\|_\infty \leq 4$, $S \leq 12$ and $\log B \lesssim \log (1/\eps) + \log m$.
\end{Co}

The proof of Corollary \ref{co:multi_approx_gelu} can be found in Appendix \ref{sec:co_multi_approx_gelu_proof}.
Having derived the approximation guarantees for multiplication, we now turn to the approximation of multiple number multiplications, as outlined in the following lemma.

\begin{Lem}[approximating the multiplication of $d$ numbers]
    \label{lem:mul_d_gelu_approx}
    Let $d, m \in \N$ with $d \geq 2$ be arbitrary, and define the function $\mathrm{prod}_d : (x_1, \dots, x_d) \mapsto \prod_{i=1}^d x_i$.
    Then, for every $\eps \in (0, 1)$ and every $K \geq 1$, there exists $\varphi_{mul, d} \in \NN(L, W, S, B)$ such that
    \begin{align*}
        (i) &\quad \|\varphi_{mul, d} - \mathrm{prod}_d\|_{W^{m, \infty}([-K, K]^d)} \leq \eps, \\
        (ii) &\quad \|\varphi_{mul, d}\|_{W^{m, \infty}(\R^d)} \leq \exp\{\cO((m^2 + d)\log(mdK \log(1 / \eps)))\} .
    \end{align*}
    In addition,
    \begin{align*}
        L \lesssim \log d,
        \quad \|W\|_\infty \vee S \lesssim d^2,
        \quad \log B \lesssim (\log(1 / \eps) + (d + m) \log K + m^2 d^2 ) \log d.
    \end{align*}
    
\end{Lem}

\begin{proof}
    To improve readability, the proof is divided into several steps.
    
    \noindent
    \textbf{Step 1: approximation error analysis.}\quad
    We first prove the statement for $K = 1$ and then generalize it to any arbitrary $K \geq 1$.
    Overall, the resulting neural network is structured as a binary tree, in accordance with the methodology described in \cite{schwab2019deep}.
    We build an approximation of multiplication of $2^{J}$ numbers with $J = \ceil{\log_2 d}$.
    If $d < 2^J$, then a minor modification of the input layer implements a concatenation of the input vector with the vector of ones of length at most $d$.
    Now let
    \begin{align}
        \label{eq:phi_j_comp_def}
        &\varphi_{j}(x_{1:2^j}) = \varphi_{mul, j}\left(\varphi_{j - 1, 1}(x_{1:2^{j-1}}), \varphi_{j - 1, 2}(x_{2^{j-1} + 1: 2^j})\right), \quad 1 \leq j \leq J,
    \end{align}
    where $\varphi_{mul, j}$ is the neural network from Corollary \ref{co:multi_approx_gelu} with accuracy parameter $\eps_{mul}^{(j)}$ and the smoothness parameter $m$, $\varphi_{j - 1, 1}$ and $\varphi_{j - 1, 2}$ are identical copies of $\varphi_{j-1}$, and $\varphi_0$ represents the identity mapping.
    From Corollary \ref{co:multi_approx_gelu} we deduce that
    \begin{align}
        \label{eq:phi_mul_j_acc}
        \|\varphi_{mul, j} - \mathrm{prod}_2\|_{W^{m, \infty}}([-C, C]^2) \leq C^3 \eps_{mul}^{(j)}
        \quad \text{for all } C \geq 1 \text{ and } 1\leq j \leq J.
    \end{align}
    To simplify the notation, we let $\varphi_j(x_{1:2^j}) = (\varphi_{mul, j} \circ (\varphi_{j - 1, 1}, \varphi_{j - 1, 2})) (x_{1:2^j})$.
    Now assume that for all $0 \leq j \leq J$ and $C \geq 1$, it holds that
    \begin{align}
        \label{eq:eps_C_j_def}
        \|\varphi_j - \mathrm{prod}_{2^j}\|_{W^{m, \infty}(\Omega_j)} = \eps_j,
    \end{align}
    where $\Omega_j = [-1, 1]^{2^j}$.
    Hence, for any $1 \leq j \leq J$, the triangle inequality suggests that
    \begin{align}
        \label{eq:eps_j_decomp}
        \notag
        &\|\varphi_j - \mathrm{prod}_{2^j}\|_{W^{m, \infty}(\Omega_j)} \\
        \notag
        &\quad =
        \|\varphi_{mul, j} \circ (\varphi_{j - 1, 1}, \varphi_{j - 1, 2}) - \mathrm{prod}_{2^j} \|_{W^{m, \infty}(\Omega_j)} \\
        &\quad \leq
        \|(\varphi_{mul, j} - \mathrm{prod}_2) \circ (\varphi_{j - 1, 1}, \varphi_{j - 1, 2})\|_{W^{m, \infty}(\Omega_j)}
        + \|\varphi_{j - 1, 1} \cdot \varphi_{j - 1, 2} - \mathrm{prod}_{2^j}\|_{W^{m, \infty}(\Omega_j)} .
    \end{align}
    As for the first term of \eqref{eq:eps_j_decomp}, we apply Lemma \ref{lem:comp_sob_norm} and arrive at
    \begin{align}
        \label{eq:mul_min_prod_circ_phi}
        \notag
        &\|(\varphi_{mul, j} - \mathrm{prod}_2) \circ (\varphi_{j - 1, 1}, \varphi_{j - 1, 2})\|_{W^{m, \infty}(\Omega_j)} \\
        \notag
        &\quad \leq 16(e^2 m^4 \cdot 2 \cdot 4^{j - 1})^m \|\varphi_{mul, j} - \mathrm{prod}_2\|_{W^{m, \infty}([-1 - \eps_{j - 1}, 1 + \eps_{j - 1}]^2)} (1 \vee \|\varphi_{j - 1, 1}\|^m_{W^{m, \infty}(\Omega_{j-1})}) \\
        &\quad \leq 16(e^2 m^4 \cdot 2 \cdot 4^{j - 1})^m (1 + \eps_{j - 1})^{m + 3} \eps_{mul}^{(j)},
    \end{align}
    where the last inequality uses \eqref{eq:phi_mul_j_acc}.
    As for the second term of \eqref{eq:eps_j_decomp}, we apply Lemma \ref{lem:prod_sob_norm} and obtain that
    \begin{align*}
        \|\varphi_{j - 1, 1} \cdot \varphi_{j - 1, 2} - \mathrm{prod}_{2^j} \|_{W^{m, \infty}(\Omega_j)} 
        &\leq 2^m \| \varphi_{j - 1, 1} - \mathrm{prod}_{2^{j - 1}} \|_{W^{m, \infty}(\Omega_{j - 1})} \| \varphi_{j - 1, 2} \|_{W^{m, \infty}(\Omega_{j - 1})} \\
        &\quad + 2^m\|\mathrm{prod}_{2^{j - 1}}\|_{W^{m, \infty}(\Omega_{j - 1})} \|\varphi_{j - 1, 2} - \mathrm{prod}_{2^{j - 1}}\|_{W^{m, \infty}(\Omega_{j - 1})} .
    \end{align*}
    From \eqref{eq:eps_C_j_def} we deduce that
    \begin{align}
        \label{eq:phi_1_phi_2_min_prod}
        \|\varphi_{j - 1, 1} \cdot \varphi_{j - 1, 2} - \mathrm{prod}_{2^j} \|_{W^{m, \infty}(\Omega_j)}
        \leq 2^{m + 1} \eps_{j - 1} (1 + \eps_{j - 1}) .
    \end{align}
    Therefore, combining \eqref{eq:eps_j_decomp}, \eqref{eq:mul_min_prod_circ_phi} and \eqref{eq:phi_1_phi_2_min_prod}, we arrive at
    \begin{align*}
        \eps_j
        \leq 16(e^2 m^4 \cdot 2 \cdot 4^{j - 1})^m (1 + \eps_{j - 1})^{m + 3} \eps_{mul}^{(j)}
        + 2^{m + 1} \eps_{j - 1} (1 + \eps_{j - 1}) .
    \end{align*}
    We find $\eps_j$ in the form of $\eps_j = 2^{\gamma_j}\eps_1$ for each $1 \leq j \leq J$ with $\gamma_1 = 0$.
    Hence,
    \begin{align*}
        \eps_j
        \leq 16(e^2 m^4 \cdot 2 \cdot 4^{j - 1})^m 2^{(m + 1)(\gamma_{j - 1} + 1)} \eps_{mul}^{(j)}
        + 2^{m + 1} 2^{2\gamma_{j - 1} + 1} \eps_1.
    \end{align*}
    Setting 
    \begin{align}
    \label{eq:eps_mul_j}
        \eps_{mul}^{(j)}
        = \eps_1 \left(16(e^2 m^4 \cdot 2 \cdot 4^{j - 1})^m 2^{(m + 1)(\gamma_{j - 1} + 1)}\right)^{-1}, \quad 2 \leq j \leq J,
    \end{align}
    we have that
    \begin{align*}
        \gamma_j \leq m + 3 + 2\gamma_{j - 1}, \quad 2 \leq j \leq J,
    \end{align*}
    which yields that
    \begin{align*}
        \gamma_j \leq (m + 3)4^j, \quad 1 \leq j \leq J.
    \end{align*}
    Therefore,
    \begin{align*}
        \eps_J
        \leq (m + 3) 2^{2\ceil{\log_2 d}}\eps_1
        \leq 4(m + 3)d^2 \eps_1.
    \end{align*}
    Choosing $\eps_{mul}^{(1)} = \eps (4(m + 3)d^2)^{-1} \in (0, 1)$ ensures that
    \begin{align*}
        \|\varphi_{J} - \mathrm{prod}_{2^J}\|_{W^{m, \infty}([-C, C])}
        \leq \eps.
    \end{align*}

    \noindent
    \textbf{Step 2: deriving the configuration of $\varphi_J$.}\quad
    Due to the observation that $\gamma_j \lesssim m d^2$, we deduce from Corollary \ref{co:multi_approx_gelu} and \eqref{eq:eps_mul_j} that, for all $1 \leq j \leq J$, we have $\varphi_{mul, j} \in \NN(L_{mul}, W_{mul}, S_{mul}, B_{mul})$ with
    \begin{align*}
        L_{mul} = 2, \quad \|W_{mul}\|_\infty \vee S_{mul} \lesssim 1,
        \quad \log B_{mul} \lesssim \log(1 / \eps) + m^2 d^2.
    \end{align*}
    Let $\varphi_j \in \NN(L_j, W_j, S_j, B_j)$ for all $1 \leq j \leq J$.
    Then, from Lemma \ref{lem:concat_nn} and Lemma \ref{lem:paral_nn} we find that
    \begin{align}
        \label{eq:phi_J_mul_cfg_aux}
        \notag
        &L_J \leq J + 1,
        \quad \|W_J\|_\infty \vee S_J \lesssim 2^{J - 1}, \\
        &\log B_J \leq \log B_{J - 1} + \log B_{mul} + \log \|W_{J - 1}\|_\infty
        \lesssim (\log(1 / \eps) + m^2 d^2) \log d.
    \end{align}
    We now generalize the approximation result to the case when $K \geq 1$.
    Let
    \begin{align*}
        \varphi_{J, K}(x_1, \dots, x_d) = K^d \varphi_J(x_1 / K, \dots, x_d / K).
    \end{align*}
    Then, from the chain rule we obtain that
    \begin{align*}
        \|\varphi_{J, K} - \mathrm{prod}_d\|_{W^{m, \infty}([-K, K]^d)}
        \leq K^d \| \varphi_J - \mathrm{prod}_d \|_{W^{m, \infty}([-1, 1]^d)}
        \leq K^d \eps.
    \end{align*}
    Therefore, taking the accuracy parameter $\eps / K^d$ in $\varphi_{J, K}$, we deduce that for any $\eps \in (0, 1)$ there exists $\tilde{\varphi}_{mul, d, K} \in \NN(\tilde{L}, \tilde{W}, \tilde{S}, \tilde{B})$ satisfying
    \begin{align*}
        \|\tilde{\varphi}_{mul, d, K} - \mathrm{prod}_d \|_{W^{m, \infty}([-K, K]^d)} \leq \eps.
    \end{align*}
    Furthermore, \eqref{eq:phi_J_mul_cfg_aux} we find that
    \begin{align}
        \label{eq:tilde_phi_mul_cfg}
        \tilde{L} \lesssim \log d,
        \quad \|\tilde{W}\|_\infty \vee \tilde{S} \lesssim d,
        \quad \log \tilde{B} \lesssim (\log(1 / \eps) + d \log K + m^2 d^2 ) \log d.
    \end{align}

    \noindent
    \textbf{Step 3: clipping the input.}\quad
    Now let $\varphi_{clip}$ be a clipping operation approximation from Lemma \ref{lem:clip_gelu_approx} formulated with accuracy parameter $\eps_{clip} \in (0, 1)$ and clipping parameter $K$.
    Then, it holds that $\|\varphi_{clip}\|_{W^{0, \infty}(\R)} \leq K + 5 / 2 \leq 4 K$.
    Let $\varphi_{clip, d}$ be a parallel stacking of $d$ identical copies of $\varphi_{clip}$ that approximates a component-wise clipping.
    Let also $\tilde{\varphi}_{mul, d, 4K}$ has accuracy parameter $\eps_{mul, d}$ and smoothness parameter $m + 1$.
    Then, it holds that
    \begin{align*}
        &\|\tilde{\varphi}_{mul, d, 4K} \circ \varphi_{clip, d} - \mathrm{prod}_d \circ \id \|_{W^{m, \infty}([-K, K]^d)} \\
        &\quad \leq \| (\tilde{\varphi}_{mul, d, 4K} - \mathrm{prod}_d) \circ \varphi_{clip, d} \|_{W^{m, \infty}([-K, K]^d)}
        + \|\mathrm{prod}_d \circ \varphi_{clip, d} - \mathrm{prod}_d \circ \id \|_{W^{m, \infty}([-K, K]^d)}.
    \end{align*}
    Lemma \ref{lem:comp_sob_norm} suggests that
    \begin{align*}
        \| (\tilde{\varphi}_{mul, d, 4K} - \mathrm{prod}_d) \circ \varphi_{clip, d} \|_{W^{m, \infty}([-K, K]^d)}
        &\leq \exp\{\cO(m \log (m d))\} \eps_{mul, d} (\eps_{clip} + K)^m \\
        &\leq \exp\{ \cO(m \log(m d K) ) \} \eps_{mul, d}
    \end{align*}
    and also
    \begin{align*}
        \|\mathrm{prod}_d \circ \varphi_{clip, d} - \mathrm{prod}_d \circ \id \|_{W^{m, \infty}([-K, K]^d)}
        &\leq \exp\{ \cO(m\log(md)) \} (K + 5 / 2)^d \eps_{clip} (\eps_{clip} + K)^{2m} \\
        &\leq \exp\{ \cO(m\log(mdK) + d\log K) \}\eps_{clip}.
    \end{align*}
    Therefore, setting
    \begin{align}
        \label{eq:mul_d_clip_eps}
        \log(1 / \eps_{mul, d}) \asymp \log(1 / \eps) + m \log(mdK),
        \quad \log(1 / \eps_{clip}) \asymp \log(1 / \eps) + m \log(mdK) + d \log K
    \end{align}
    for some $\eps \in (0, 1)$ ensures that
    \begin{align*}
        \|\tilde{\varphi}_{mul, d, 4K} \circ \varphi_{clip, d} - \mathrm{prod}_d \circ \id \|_{W^{m, \infty}([-K, K]^d)}
        \leq \eps.
    \end{align*}
    Moreover, Lemma \ref{lem:clip_gelu_approx} and Lemma \ref{lem:comp_sob_norm} imply that 
    \begin{align*}
        \|\tilde{\varphi}_{mul, d, 4K} \circ \varphi_{clip, d}\|_{W^{m, \infty}(\R^d)}
        &\leq \exp\{\cO(m \log(md))\} (\eps_{mul, d} + (4K)^d) (1 \vee \|\varphi_{clip, d}\|_{W^{m, \infty}(\R^d)}^m) \\
        &\leq \exp\{\cO((m^2 + d)\log(mdK \log(1 / \eps)))\} ,
    \end{align*}
    where the last inequality uses \eqref{eq:mul_d_clip_eps}.
    Recall that due to Lemma \ref{lem:clip_gelu_approx}, Lemma \ref{lem:paral_nn} and \eqref{eq:mul_d_clip_eps}, we have that $\varphi_{clip, d} \in \NN(L_{clip}, W_{clip}, S_{clip}, B_{clip})$ with 
    \begin{align*}
        L_{clip} \lesssim 1,
        \quad \|W_{clip}\|_\infty \vee S_{clip} \lesssim d,
        \quad \log B_{clip} \lesssim \log(1 / \eps) + m \log(mdK) + d \log K.
    \end{align*}
    Finally, from Lemma \ref{lem:concat_nn}, \eqref{eq:tilde_phi_mul_cfg} and \eqref{eq:mul_d_clip_eps} we deduce that $\varphi_{mul, d} = \tilde{\varphi}_{mul, d, 4K} \circ \varphi_{clip, d}$ has
    \begin{align*}
        L \lesssim \log d,
        \quad \|W\|_\infty \vee S \lesssim d^2,
        \quad \log B \lesssim (\log(1 / \eps) + (d + m) \log K + m^2 d^2 ) \log d.
    \end{align*}
    The proof is complete.
    
\end{proof}

Comparing our result from Lemma \ref{lem:mul_d_gelu_approx} to that presented in \cite[Corollary 3.8]{de2021approximation}, we observe a difference in the number of parameters: $\mathcal{O}(d^2)$ versus $\mathcal{O}(d \log d)$.
We emphasize that, as a byproduct, we derived a neural network with the number of parameters $\cO(d)$, but we employed clipping and concatenation to satisfy condition $(ii)$, which ultimately increased the parameter count.
However, by adding clipping, we ensure that the approximation and its derivatives are bounded across the entire real line.
Now, we turn to the approximation of monomials, as formulated in the following lemma.

\begin{Co}[approximation of monomials]
    \label{co:monom_approx_gelu}
    Let $\bk \in \Z_+^I$ for some $I \in \N$ such that $|\bk| = d$, where $d \in \N$ with $d \geq 2$ is arbitrary.
    Define $\mathrm{prod}_\bk : (x_1, \dots, x_I) \mapsto \prod_{i = 1}^I x_i^{k_i}$.
    Then, for every $\eps \in (0, 1)$, every $m \in \N$, and every $K \geq 1$, there exists a GELU network $\varphi_{mul, \bk} \in \NN(L, W, S, B)$ such that
    \begin{align*}
        (i) &\quad \|\varphi_{mul, \bk} - \mathrm{prod}_\bk \|_{W^{m, \infty}([-K, K]^I)} \leq \eps, \\
        (ii) &\quad \|\varphi_{mul, \bk}\|_{W^{m, \infty}(\R^I)} \leq \exp\{\cO((m^2 + d)\log(mdK \log(1 / \eps)))\}.
    \end{align*}
    In addition, $\varphi_{mul, \bk}$ has
    \begin{align*}
        L \lesssim \log d,
        \quad \|W\|_\infty \vee S \lesssim (d \vee I)^3,
        \quad \log B \lesssim (\log(1 / \eps) + (d + m) \log K + m^2 d^2 ) \log d + \log I.
    \end{align*}

\end{Co}

We move the proof of Corollary \ref{co:monom_approx_gelu} to Appendix \ref{sec:co_monom_approx_gelu_proof}.
The following lemma provides an approximation result for multivariate polynomials.

\begin{Lem}[approximation of multivariate polynomials]
    \label{lem:part_sum_monomial}
    Define $f_{\mathcal{A}} : x \mapsto \sum_{\bk \in \mathcal{A} } a_\bk x^\bk$, where $x \in \R^I$ and $\mathcal{A} = \{\bk \in \Z_+^I \; : \; |\bk| \leq d \}$ for some $I, d \in \N$ with $d \geq 2$.
    Also assume that $|a_\bk| \leq 1$ for all $\bk \in \mathcal{A}$.
    Then, for every $\eps \in (0, 1)$, every natural $m \geq 3$ and every $K \geq 1$, there exists a neural network $\varphi_{\mathcal{A}} \in \NN(L, W, S, B)$ such that
    \begin{align*}
        (i) &\quad \|f_\mathcal{A} - \varphi_\mathcal{A}\|_{W^{m, \infty}([-K, K]^I)} \leq \eps, \\
        (ii) &\quad \|\varphi_\mathcal{A}\|_{W^{m, \infty}(\R^I)} \leq \exp\{\cO( (m^2 + md + I)\log(mdKI\log(1 / \eps)) )\}.
    \end{align*}
    In addition, $\varphi_{\mathcal{A}}$ has
    \begin{align*}
        &L \lesssim \log d,
        \quad \|W\|_\infty \vee S \lesssim (d + I)^{3 + d \wedge I} \\
        &\log B \lesssim (\log(1 / \eps) + m^2(d + I)\log(mdKI) + m^2d^2)\log(d + I).
    \end{align*}
    
\end{Lem}

\begin{proof}
    Corollary \ref{co:monom_approx_gelu} implies that for each $\bk \in \mathcal{A}$ with $|\bk| \geq 2$ there exists $\varphi_\bk$ satisfying
    \begin{align}
        \label{eq:exp_phi_alpha_approx}
        \|\varphi_\bk - \mathrm{prod}_\bk \|_{W^{m, \infty}([-K, K]^I)}
        \leq \eps_\bk,
    \end{align}
    where $\eps_\bk \in (0, 1)$ is accuracy parameter.
    Moreover,
    \begin{align}
        \label{eq:phi_alpha_bound}
        \|\varphi_{\bk}\|_{W^{m, \infty}(\R^I)} \leq \exp\{\cO((m^2 + d)\log(mdK \log(1 / \eps_\bk)))\}
    \end{align}
    and $\varphi_\bk \in \NN(L_\bk, W_\bk, S_\bk, B_\bk)$ with
    \begin{align}
        \label{eq:phi_alpha_cfg}
        \notag
        &L_\bk \lesssim \log d,
        \quad \|W_\bk\|_\infty \vee S_\bk \lesssim (d \vee I)^3, \\
        &\log B_\bk \lesssim (\log(1 / \eps_\bk) + (d + m) \log K + m^2 d^2 ) \log d + \log I.
    \end{align}
    As for $|\bk| \in \{0, 1\}$, the approximation is exact, since it is implemented with a single linear layer.
    In order to build the final approximation, we have to implement a summation of GELU networks with different depth.
    For this purpose, we add auxiliary identity layers.
    Let $\varphi_{id, \bk}$ be an approximation of identity function from Lemma \ref{lem:id_deep_gelu_approx} formulated with the accuracy parameter $\eps_\bk$, $L_{id, \bk} = 1 + \max_{\tilde{\bk} \in \mathcal{A}}L_{\tilde{\bk}} - L_\bk$ number of layers and the scale parameter $\|\varphi_\bk\|_{W^{0, \infty}([-K, K]^I)}$.
    Hence, the triangle inequality implies that
    \begin{align}
        \label{eq:phi_id_circ_phi_alpha_acc_aux}
        \notag
        &\| \varphi_{id, \bk} \circ \varphi_\bk - \mathrm{prod}_{\bk} \|_{W^{m, \infty}([-K, K]^I)} \\
        &\quad \leq \|\varphi_\bk - \mathrm{prod}_\bk\|_{W^{m, \infty}([-K, K]^I)}
        + \|(\varphi_{id, \bk} - \id) \circ \varphi_\bk\|_{W^{m, \infty}([-K, K]^I)}.
    \end{align}
    Therefore, Lemma \ref{lem:comp_sob_norm} suggest that
    \begin{align*}
        &\|(\varphi_{id, \bk} - \id) \circ \varphi_\bk\|_{W^{m, \infty}([-K, K]^I)} \\
        &\quad \leq \exp\{\cO(m\log m)\} \|\varphi_{id, \bk} - \id\|_{W^{m, \infty}(\Omega_2)}(\|\varphi_\bk\|^m_{W^{m, \infty}([-K, K]^I)} \vee 1),
    \end{align*}
    where $\Omega_2 = [-\|\varphi_\bk\|_{W^{0, \infty}([-K, K]^I)}, \|\varphi_\bk\|_{W^{0, \infty}([-K, K]^I)} ]$.
    As suggested by \eqref{eq:exp_phi_alpha_approx}, we have that
    \begin{align}
        \label{eq:phi_alpha_K_bound}
        \|\varphi_\bk\|_{W^{0, \infty}([-K, K]^I)}
        \leq \|\varphi_\bk\|_{W^{m, \infty}([-K, K]^I)}
        \leq \|\mathrm{prod}_{\bk}\|_{W^{m, \infty}([-K, K]^I)} + \eps_\bk
        \leq 2 d^m K^d.
    \end{align}
    Hence, it holds that
    \begin{align*}
        \|(\varphi_{id, \bk} - \id) \circ \varphi_\bk\|_{W^{m, \infty}([-K, K]^I)}
        \leq \exp\{\cO(m^2 d\log(mdK))\}\eps_\bk.
    \end{align*}
    Then, from \eqref{eq:exp_phi_alpha_approx} and \eqref{eq:phi_id_circ_phi_alpha_acc_aux} we deduce that
    \begin{align}
        \label{eq:phi_id_circ_phi_alpha_acc}
        \| \varphi_{id, \bk} \circ \varphi_\bk - \mathrm{prod}_{\bk} \|_{W^{m, \infty}([-K, K]^I)}
        \leq \exp\{\cO(m^2 d\log(mdK))\}\eps_\bk, \quad \text{for all } \bk \in \mathcal{A}.
    \end{align}
    We now specify the configuration of each $\varphi_{id, \bk}$, as suggested by \eqref{eq:phi_alpha_K_bound} and Lemma \ref{lem:id_deep_gelu_approx}:
    \begin{align*}
        L(\varphi_{id, \bk}) &\lesssim \log d,
        \quad \|W\|_\infty(\varphi_{id, \bk}) \lesssim 1, \\
        S(\varphi_{id, \bk}) &\lesssim \log d, \quad 
        \log B(\varphi_{id, \bk}) \lesssim (m + \log d)\log m + \log(1 / \eps_\bk) + m^2 \log d + d m \log K.
    \end{align*}
    Now \eqref{eq:phi_alpha_cfg} and Lemma \ref{lem:concat_nn} imply that the composition $\varphi_{\bk} \circ \varphi_{id, \bk}$ has
    \begin{align*}
        &L(\varphi_{\bk} \circ \varphi_{id, \bk}) \lesssim \log d,
        \quad \|W\|_\infty(\varphi_{\bk} \circ \varphi_{id, \bk}) \vee S(\varphi_{\bk} \circ \varphi_{id, \bk}) \lesssim (d \vee I)^3, \\
        &\log B(\varphi_{\bk} \circ \varphi_{id, \bk}) \lesssim (\log(1 / \eps_\bk) + d m \log K + m^2 d^2 ) \log d + \log I.
    \end{align*}
    Now setting $\eps_\bk = \eps / |\mathcal{A}|$ and applying parallelization argument from Lemma \ref{lem:paral_nn}, for
    \begin{align*}
        \varphi_{\mathcal{A}}(x) = \sum_{\bk \in \mathcal{A}} a_\bk \cdot (\varphi_{id, \bk} \circ \varphi_\bk)(x),
        \quad x\in \R^I,
    \end{align*}
    we obtain that $\varphi_{\mathcal{A}}$ has
    \begin{align}
        \label{eq:phi_A_cfg_aux}
        \notag
        &L(\varphi_{\mathcal{A}}) \lesssim \log d,
        \quad \|W\|_\infty(\varphi_{\mathcal{A}}) \vee S(\varphi_{\mathcal{A}}) \lesssim |\mathcal{A}| (d \vee I)^3 \\
        &\log B(\varphi_{\mathcal{A}}) \lesssim \log (|\mathcal{A}|) + (\log(1 / \eps_\bk) + dm \log K + m^2 d^2 ) \log d + \log I.
    \end{align}
    Since
    \begin{align}
        \label{eq:A_set_bound}
        |\mathcal{A}|
        \leq \binom{d + I}{d}
        \leq (d + I)^{d \wedge I}
        = \exp\{(d \wedge I)\log (d + I)\},
    \end{align}
    then \eqref{eq:phi_id_circ_phi_alpha_acc} yields
    \begin{align*}
        \|f_{\mathcal{A}} - \varphi_{\mathcal{A}}\|_{W^{m, \infty}([-K, K]^I)}
        \leq \exp\{\cO( m^2 (d + I)\log(mdKI) )\} \eps_\bk.
    \end{align*}
    Thus, setting
    \begin{align}
        \label{eq:eps_alpha_def}
        \log(1 / \eps_\bk) = \log(1 / \eps) + m^2 (d + I)\log(mdKI)
    \end{align}
    ensures that
    \begin{align*}
        \|f_{\mathcal{A}} - \varphi_{\mathcal{A}}\|_{W^{m, \infty}([-K, K]^I)}
        \leq \eps.
    \end{align*}
    Moreover, the configuration described in \eqref{eq:phi_A_cfg_aux} is now
    \begin{align*}
        &L(\varphi_{\mathcal{A}}) \lesssim \log d,
        \quad \|W\|_\infty(\varphi_{\mathcal{A}}) \vee S(\varphi_{\mathcal{A}}) \lesssim (d + I)^{3 + d \wedge I} \\
        &\log B(\varphi_{\mathcal{A}}) \lesssim (\log(1 / \eps) + m^2(d + I)\log(mdKI) + m^2d^2)\log(d + I).
    \end{align*}
    From \eqref{eq:phi_alpha_bound}, \eqref{eq:A_set_bound} and Lemma \ref{lem:comp_sob_norm} we deduce that
    \begin{align*}
        \|\varphi_\mathcal{A}\|_{W^{m, \infty}(\R^I)}
        &\leq |\mathcal{A}| \max_{\bk \in \mathcal{A}}\|\varphi_{id, \bk} \circ \varphi_\bk\|_{W^{m, \infty}(\R^I)} \\
        &\leq \exp\{\cO( (d \wedge I)\log (d + I) + m\log m)\}\max_{\bk \in \mathcal{A}} \|\varphi_{id, \bk}\|_{W^{m, \infty}(\varphi_\bk(\R^I))}
        (1 \vee \|\varphi_\bk\|_{W^{m, \infty}(\R^I)}^m).
    \end{align*}
    As suggested by \eqref{eq:phi_alpha_bound} and \eqref{eq:eps_alpha_def}, we have that
    \begin{align*}
        \|\varphi_\bk\|_{W^{m, \infty}(\R^I)}
        \leq \exp\{\cO(m^2 + d)\log(mdKI\log(1 / \eps))\}.
    \end{align*}
    From Lemma \ref{lem:id_deep_gelu_approx} and \eqref{eq:phi_alpha_K_bound} we find that
    \begin{align*}
        \|\varphi_{id, \bk}\|_{W^{m, \infty}(\varphi_\bk(\R^I))}
        \leq \|\varphi_{id, \bk}\|_{W^{m, \infty}(\R)}
        \leq \exp\{ \cO(m\log(d m\log(1 / \eps_\bk)) + d \log K) \}.
    \end{align*}
    Therefore, due to \eqref{eq:eps_alpha_def}, it holds that
    \begin{align*}
        \|\varphi_{\mathcal{A}}\|_{W^{m, \infty}(\R^I)}
        \leq \exp\{\cO( (m^2 + md + I)\log(mdKI\log(1 / \eps)) )\}.
    \end{align*}
    The proof is complete.
    
\end{proof}

\subsection{Approximation of the exponent and the division}
\label{subsec:exp_div_approx}

Now, we address the approximation of nonlinear operations, including exponentiation and division.
The following lemma provides quantitative bounds for the exponential function approximation.

\begin{Lem}[approximation of the exponential function]
    \label{lem:min_exp_gelu_approx}
    Define $f_{exp} : x \mapsto e^{-x}$, and let $m \in \N$ be arbitrary.
    Then, for any $\eps \in (0, 1)$ and $0 \leq A \leq 1$, there exists a neural network $\varphi_{exp} \in \NN(L, W, S, B)$ such that
    \begin{align*}
        (i) &\quad \|\varphi_{exp} - f_{exp}\|_{W^{m, \infty}([-A, +\infty))} \leq \eps, \\
        (ii) &\quad \|\varphi_{exp}\|_{W^{m, \infty}(\R)} \leq \exp\{\cO(m^2\log(m\log(1 / \eps))  )\} .
    \end{align*}
    Furthermore, 
    \begin{align*}
        L \lesssim \log m + \log\log(1 / \eps),
        \quad \|W\|_\infty \vee S \lesssim m^{12} \log^4(1 / \eps), 
        \quad \log B \lesssim m^{11} \log^3(1 / \eps).
    \end{align*}
\end{Lem}

\begin{proof}
    For some $r \in \N$ with $r \geq 3$ and $K \geq 2$, which will be determined later, consider the approximation accuracy of a Tailor expansion
    $f_{r}(x) = \sum_{i = 0}^{r - 1}\frac{(-1)^i x^i}{i!}$
    \begin{align*}
        \|f_{exp} - f_r \|_{W^{m, \infty}([-4A, 4K])}
        \leq \max_{0 \leq m' \leq m}\left( \frac{e^{4A} (4K)^{r - m'}}{(r - m')!}\right)
        \leq e^{4A} \max_{0 \leq m' \leq m} \left(\frac{4e (K \vee A)}{r - m'}\right)^{r - m'},
    \end{align*}
    where in the last inequality we used Stirling's approximation for the factorial.
    Thus, setting
    \begin{align*}
        r = \ceil{m + 4K e^2 + 4A + \log(2 / \eps_0)} \geq 3,
    \end{align*}
    where $\eps_0 \in (0, 1)$ and will be optimized further in the proof.
    Next, we obtain that
    \begin{align}
        \label{eq:f_exp_f_r_approx}
        \|f_{exp} - f_r \|_{W^{m, \infty}([-4A, 4K])} \leq \eps_0 / 2.
    \end{align}
    Now Lemma \ref{lem:part_sum_monomial} suggests that there exists a GELU network $\tilde{\varphi}_{exp}$ formulated with the accuracy parameter $\eps_0 / 2$, the scaling parameter $4K$, the smoothness parameter $m + 1$, and the maximum power of the monomial $r - 1$ such that
    \begin{align*}
        \|f_r - \tilde{\varphi}_{exp}\|_{W^{m, \infty}([-4A, 4K])}
        \leq \|f_r - \tilde{\varphi}_{exp}\|_{W^{m, \infty}([-4K, 4K])}
        \leq \eps_0 / 2,
    \end{align*}
    which together with \eqref{eq:f_exp_f_r_approx} immediately implies that
    \begin{align}
        \label{eq:f_exp_tilde_phi_exp_acc}
        \|f_{exp} - \tilde{\varphi}_{exp}\|_{W^{m, \infty}([-4A, 4K])}
        \leq \|f_{exp} - f_r\|_{W^{m, \infty}([-4A, 4K])} + \|f_r - \tilde{\varphi}_{exp}\|_{W^{m, \infty}([-4A, 4K])}
        \leq \eps_0.
    \end{align}
    In addition, \eqref{eq:f_exp_tilde_phi_exp_acc} suggests that
    \begin{align}
        \label{eq:tilde_phi_exp_bound}
        \|\tilde{\varphi}_{exp}\|_{W^{m, \infty}([-4A, 4K])}
        &\leq \|\tilde{\varphi}_{exp} - f_{exp} \|_{W^{m, \infty}([-4A, 4K])} + \|f_{exp}\|_{W^{m, \infty}([-4A, 4K])}
        \leq 2e^{4A}.
    \end{align}
    Next, substituting the choice of $r$ into the configuration of $\tilde{\varphi}_{exp} \in \NN(\tilde{L}, \tilde{W}, \tilde{S}, \tilde{B})$ outlined in Lemma \ref{lem:part_sum_monomial} yields
    \begin{align}
        \label{eq:tilde_phi_exp_cfg}
        \notag
        &\tilde{L} \lesssim \log r \lesssim \log(m + K + \log(1 / \eps_0)),
        \quad \|\tilde{W}\|_\infty \vee \tilde{S} \lesssim r^4 \lesssim m^4 + K^4 + \log^4(1 / \eps_0), \\
        &\log \tilde{B} \lesssim (\log(1 / \eps_0) + m^2 r\log(mrK) + m^2 r^2) \log r
        \lesssim m^2(m^3 + K^3 + \log^3(1 / \eps_0)).
    \end{align}
    Let $\varphi_{clip}$ be an approximation of clipping operation from Lemma \ref{lem:clip_gelu_approx} with the accuracy parameter $\eps_{clip} \in (0, 1)$ and the scale parameter $(A + K) / 2 \geq 1$.
    Then we have that
    \begin{align*}
        \varphi_{clip, -A, K}(x) = \varphi_{clip}(x + (A - K) / 2) + (K - A) / 2
    \end{align*}
    has the following properties:
    \begin{align*}
        (i) &\quad \|\varphi_{clip, -A, K} - \id\|_{W^{m, \infty}([-A, K])}
        \leq \|\varphi_{clip} - \id\|_{W^{m, \infty}([-(A + K) / 2, (A + K) / 2])} \leq \eps_{clip}, \\
        (ii) &\quad -5/2 - A \leq \varphi_{clip, -A, K}(x) \leq K + 5 / 2, \quad \text{for all } x \in \R \\
        (iii) &\quad \|\varphi_{clip, -A, K} - K - 1/2\|_{W^{0, \infty}([K, +\infty))} \leq \eps_{clip} + 1, \\
        (iv) &\quad \|\varphi_{clip, -A, K}\|_{W^{m, \infty}(\R)} \leq \exp\{\cO(m \log(m\log(1 / \eps_{clip})) + \log(A + K))\}.
    \end{align*}
    Hence, from properties $(i)$, $(ii)$ and Lemma \ref{lem:comp_sob_norm} we obtain for $\varphi_{exp} = \tilde{\varphi}_{exp} \circ \varphi_{clip, -A, K}$ that
    \begin{align*}
        \|\tilde{\varphi}_{exp} \circ \varphi_{clip, -A, K} - \tilde{\varphi}_{exp}\|_{W^{m, \infty}([-A, K])}
        \leq \exp\{\cO(m\log(mK))\}\|\tilde{\varphi}_{exp}\|_{W^{m + 1, \infty}([-4A, 4K])} \eps_{clip}.
    \end{align*}
    From \eqref{eq:tilde_phi_exp_bound} we find that 
    \begin{align*}
        \|\tilde{\varphi}_{exp} \circ \varphi_{clip, -A, K} - \tilde{\varphi}_{exp}\|_{W^{m, \infty}([-A, K])}
        \leq \exp\{\cO(m\log(mK))\} \eps_{clip}.
    \end{align*}
    Thus, \eqref{eq:f_exp_tilde_phi_exp_acc} implies that for
    \begin{align}
        \label{eq:phi_exp_eps_clip}
        \log(1 / \eps_{clip}) \asymp \log(1 / \eps_0) + m\log mK
    \end{align}
    we have
    \begin{align}
        \label{eq:phi_exp+f_exp_A_K}
        \|\varphi_{exp} - f_{exp}\|_{W^{m, \infty}([-A, K])} \leq 2\eps_0.
    \end{align}
    Property $(iii)$ together with Lemma \ref{lem:comp_sob_norm} suggests that
    \begin{align*}
        &\| \tilde{\varphi}_{exp} \circ \varphi_{clip, -A, K} \|_{W^{m, \infty}([K, +\infty))} \\
        &\quad \leq \exp\{\cO(m\log m)\} \|\tilde{\varphi}_{exp}\|_{W^{m, \infty}([K - 3/2, K + 5/2])} (1 \vee \|\varphi_{clip, -A, K}\|_{W^{m, \infty}(\R)}^m).
    \end{align*}
    From \eqref{eq:f_exp_tilde_phi_exp_acc}, \eqref{eq:phi_exp_eps_clip} and property $(iv)$ we find that
    \begin{align*}
        \| \tilde{\varphi}_{exp} \circ \varphi_{clip, -A, K} \|_{W^{m, \infty}([K, +\infty))}
        \leq \exp\{\cO(m^2\log(mK\log(1 / \eps_0)))\}(\eps_0 + e^{-K}).
    \end{align*}
    Thus, for $K = 2 \vee \log(1 / \eps_0)$ we have that
    \begin{align*}
        \| \tilde{\varphi}_{exp} \circ \varphi_{clip, -A, K} \|_{W^{m, \infty}([K, +\infty))}
        \leq \exp\{\cO(m^2\log(m\log(1 / \eps_0)))\}\eps_0.
    \end{align*}
    This and the triangle inequality imply that
    \begin{align*}
        \| \tilde{\varphi}_{exp} \circ \varphi_{clip, -A, K} - f_{exp} \|_{W^{m, \infty}([K, +\infty))}
        \leq \exp\{\cO(m^2\log(m\log(1 / \eps_0)))\}\eps_0.
    \end{align*}
    Therefore, from \eqref{eq:phi_exp+f_exp_A_K} we deduce that
    \begin{align*}
        \|\varphi_{exp} - f_{exp}\|_{W^{m, \infty}([-A, +\infty))} \leq \exp\{\cO(m^2\log(m\log(1 / \eps_0)))\}\eps_0 .
    \end{align*}
    Hence, setting
    \begin{align}
        \label{eq:phi_exp_eps_0_def}
        \log(1 / \eps_0) \asymp m^2 \log m + m^3\log(1 / \eps)
    \end{align}
    ensures that
    \begin{align*}
        \|\varphi_{exp} - f_{exp}\|_{W^{m, \infty}([-A, +\infty))} \leq \eps.
    \end{align*}
    Moreover, properties $(ii)$, $(iv)$ together with Lemma \ref{lem:comp_sob_norm}, \eqref{eq:tilde_phi_exp_bound}, \eqref{eq:phi_exp_eps_clip} and \eqref{eq:phi_exp_eps_0_def} yields
    \begin{align*}
        \|\tilde{\varphi}_{exp} \circ \varphi_{clip, -A, K}\|_{W^{m, \infty}(\R)}
        &\leq \|\tilde{\varphi}_{exp}\|_{W^{m, \infty}([-4A, 4K])} \exp\{\cO(m^2\log(m \log(1 / \eps_{clip})) + m\log(2K) )\} \\
        &\leq \exp\{\cO(m^2\log(m\log(1 / \eps))  )\} .
    \end{align*}
    From \eqref{eq:phi_exp_eps_clip}, \eqref{eq:phi_exp_eps_0_def} and Lemma \ref{lem:clip_gelu_approx} we find that the configuration of $\varphi_{clip} \in \NN(L_{clip}, W_{clip}, S_{clip}, B_{clip})$ is
    \begin{align*}
        L_{clip} \vee \|W_{clip}\|_\infty \vee S_{clip} \lesssim 1,
        \quad \log B_{clip} \lesssim \log(mK / \eps_{clip})
        \lesssim m^2\log m + m^3\log(1 / \eps).
    \end{align*}        
    Therefore, \eqref{eq:tilde_phi_exp_cfg}, \eqref{eq:phi_exp_eps_0_def} and Lemma \ref{lem:concat_nn} suggest that the configuration of $\varphi_{exp}$ is
    \begin{align*}
        L \lesssim \log m + \log\log(1 / \eps),
        \quad \|W\|_\infty \vee S \lesssim m^{12} \log^4(1 / \eps), 
        \quad \log B \lesssim m^{11} \log^3(1 / \eps).
    \end{align*}
    This completes the proof.
        
\end{proof}

Comparing our result presented in Lemma \ref{lem:min_exp_gelu_approx} with \cite[Corollary F.3]{yakovlev2025generalization}, we observe a less favorable configuration scaling.
Specifically, the number of parameters in Lemma \ref{lem:min_exp_gelu_approx} scales as $\cO(\log^4(1 / \eps))$, compared to $\cO(\log^2(1 / \eps))$.
Nevertheless, we extend the approximation guarantees to high-order Sobolev norms.

Now, we focus on the approximation of the division operation, beginning by approximating the reciprocal function in a straightforward manner, as suggested by the following lemma.

\begin{Lem}[naive approximation of the reciprocal function]
    \label{lem:div_approx_naive}
    Let $0 < a \leq b \leq 2$ such that $b / a \geq 5/4$ and $a < 1$.
    Define $f_{rec} : x \mapsto 1/x$, where $x > 0$.
    Then, for every $\eps \in (0, 1)$ and every $m \in \N$ such that $m \geq 3$, there exists $\varphi_{rec} \in \NN(L, W, S, B)$ satisfying
    \begin{align*}
        (i) &\quad \|\varphi_{rec} - f_{rec}\|_{W^{m, \infty}([a, b])} \leq \eps, \\
        (ii) &\quad \|\varphi_{rec}\|_{W^{m, \infty}(\R)} \leq  \exp\{\cO(m^2\log(m N / a) + m^2\log\log(1 / \eps) + m N)\} .
    \end{align*}
    Moreover, $\varphi_{rec}$ has
    \begin{align*}
        &L \lesssim \log((mb / a)\log(1 / \eps)),
        \quad \|W\|_\infty \vee S \lesssim (mb/a)^4 \log^4(m/\eps a), \\
        &\log B \lesssim (m^4 b^2 / a^2)\log^2(1 / \eps a)\log^2( (mb / a)\log(1 / \eps a)).
    \end{align*}
\end{Lem}

The proof of Lemma \ref{lem:div_approx_naive} is deferred to Appendix \ref{sec:lem_div_approx_naive_proof}.
Overall, the proof is similar to that of \cite[Lemma A.8]{yakovlev2025generalization}, but extends it to Sobolev norms.
The following result constructs a strong approximator by leveraging the weak approximators derived in Lemma \ref{lem:div_approx_naive}, drawing inspiration from \cite[Lemma A.4]{yakovlev2025generalization}.

\begin{Lem}[reciprocal function approximation]
    \label{lem:recip_gelu_approx}
    Define $f_{rec} : x \mapsto 1 / x$ for any $x > 0$.
    Let also $a_0 = 2^{-N}$ for some $N \in \N$ such that $N \geq 3$.
    Then, for every $\eps \in (0, 1)$ and every $m \in \N$ with $m \geq 3$, there exists a GELU network $\varphi_{rec} \in \NN(L, W, S, B)$ such that
    \begin{align*}
        (i) &\quad
        \|\varphi_{rec} - f_{rec}\|_{W^{m, \infty}([a_0, 1])} \leq  \eps, \\
        (ii) &\quad \|\varphi_{rec}\|_{W^{m, \infty}(\R)} \leq \exp\{\cO(m^3 N + m^3\log(m\log(1 / \eps)))\} .
    \end{align*}
    In addition, the network has 
    \begin{align*}
        &L \lesssim \log(mN\log(1 / \eps)),
        \quad \|W\|_\infty \vee S \lesssim m^8 N(N^4 + m^4\log^4(1 / \eps)), \\
        &\log B \lesssim  m^8(N^4 + m^4\log^4(1 / \eps)).
    \end{align*}
\end{Lem}

\begin{proof}
    The proof proceeds in multiple steps.
    
    \noindent
    \textbf{Step 1: introducing basic approximators.}\quad
    Let $N \in \N$ with $N \geq 3$ and let $a_i = 2^{-N + i}$ for each $i \in \{-1, 1, \dots, N + 1\}$.
    Let also $\varphi_{rec}$ be in the following form:
    \begin{align}
        \label{eq:phi_rec_pou_def}
        \varphi_{rec}(x) = \sum_{i=1}^N q(\varphi_i, \psi_i), \quad \varphi_i = \varphi_{id, i} \circ \varphi_{rec, i},
        \; \psi_i = \psi_{id, i} \circ \psi_{pou, i},
    \end{align}    
    where $q$ is a GELU network from Corollary \ref{co:multi_approx_gelu}, which approximates multiplication with accuracy parameter $\eps_{mul}$.
    Networks $\{\psi_{pou, i}\}_{i=1}^N$ form a partition of unity according to Lemma \ref{lem:pou_gelu_approx} with accuracy $\eps_{pou}$.
    In addition, the networks $\{\varphi_{rec, i}\}_{i=1}^N$ serve as local approximators of the reciprocal function from Lemma \ref{lem:div_approx_naive} with the accuracy $\eps_{rec} / 2$ and the parameters $a = a_{i - 2}$ and $b = a_{i + 1} \wedge 1$.
    Hence, we have that
    \begin{align}
        \label{eq:phi_rec_i_basic_acc}
        \max_{1 \leq i \leq N}\|\varphi_{rec, i} - f_{rec}\|_{W^{m, \infty}([a_{i - 2}, a_{i + 1} \wedge 1])} \leq \eps_{rec} / 2.
    \end{align}
    This implies that
    \begin{align}
        \label{eq:phi_rec_i_seg_bound}
        \|\varphi_{rec, i}\|_{W^{m, \infty}([a_{i - 2}, a_{i + 1} \wedge 1])}
        \leq \eps_{rec} + \exp\{\cO(m\log(m / a_0))\}
        \leq \exp\{\cO(mN + m\log m)\}
    \end{align}
    and also
    \begin{align}
        \label{eq:phi_rec_i_seg_bound_l_inf}
        \|\varphi_{rec, i}\|_{W^{0, \infty}([a_{i - 2}, a_{i + 1}\wedge 1])}
        \leq \eps_{rec} + 1 / a_{-1} \leq 4 / a_0.
    \end{align}
    The networks $\{\varphi_{id, i}\}_{i=1}^N$ approximate the identity operation (see Lemma \ref{lem:id_deep_gelu_approx}) with the accuracy parameter $\eps_{id, \varphi} \in (0, 1)$, the scale parameter $4 / a_0$.
    Similarly, the networks $\{\psi_{id, i}\}_{i=1}^N$ aim to approximate the identity operation with the accuracy parameter $\eps_{id, \psi} \in (0, 1)$ and the scale parameter $\|\psi_{pou, i}\|_{W^{m, \infty}(\R)}$.
    The number of layers of the identity networks will be specified later in the proof.
    We also put the smoothness parameter $m \in \N$ with $m \geq 3$ for all the networks.
    First, we derive the approximation accuracy of $\varphi_i$.
    Note that according to Lemma \ref{lem:div_approx_naive}, for all $i \in \{1, \dots, N\}$, we have $\varphi_{rec, i} \in \NN(L_{rec}, W_{rec}, S_{rec}, B_{rec})$ with
    \begin{align}
        \label{eq:phi_rec_i_cfg}
        \notag
        &L_{rec} \lesssim \log(m \log(1 / \eps_{rec})),
        \quad \|W_{rec}\| \vee S_{rec} \lesssim m^4 \log^4(m / \eps_{rec} a_0)
        \lesssim m^4 N^4 + m^4\log^4(m / \eps_{rec}), \\
        &\log B_{rec} \lesssim m^4 \log^4(m / \eps_{rec} a_0)
        \lesssim m^4 N^4 + m^4\log^4(m / \eps_{rec}).
    \end{align}
    Therefore, each $\varphi_{id, i}$ has at most $L_{id, \varphi} \lesssim \log(m\log(1 / \eps_{rec}))$ the number of layers.
    Now Lemma \ref{lem:id_deep_gelu_approx} in conjunction with \eqref{eq:phi_rec_i_seg_bound}, \eqref{eq:phi_rec_i_seg_bound_l_inf} and Lemma \ref{lem:comp_sob_norm} imply that
    \begin{align*}
        \max_{1 \leq i \leq N}\|\varphi_{id, i} \circ \varphi_{rec, i} - \varphi_{rec, i}\|_{W^{m, \infty}([a_{i - 2}, a_{i + 1}\wedge 1])}
        \leq \exp\{\cO(m^2 N + m^2\log m)\} \eps_{id, \varphi}.
    \end{align*}
    Hence setting
    \begin{align}
        \label{eq:eps_id_def}
        \log(1 / \eps_{id, \varphi}) \asymp \log(1 / \eps_{rec}) + m^2 N + m^2 \log m
    \end{align}
    guarantees that
    \begin{align*}
        \max_{1 \leq i \leq N}\|\varphi_{id, i} \circ \varphi_{rec, i} - \varphi_{rec, i}\|_{W^{m, \infty}([a_{i - 2}, a_{i + 1}\wedge 1])} \leq \eps_{rec} / 2.
    \end{align*}
    This and \eqref{eq:phi_rec_i_basic_acc} imply that
    \begin{align}
        \label{eq:phi_i_f_rec_acc}
        \max_{1 \leq i \leq N}\|\varphi_i - f_{rec}\|_{W^{m, \infty}([a_{i - 2}, a_{i + 1} \wedge 1])}
        \leq \eps_{rec}.
    \end{align}
    Furthermore, for $\varphi_{id, i} \in \NN(L_{id, \varphi}, W_{id, \varphi}, S_{id, \varphi}, B_{id, \varphi})$ we have from \eqref{eq:phi_rec_i_cfg}, \eqref{eq:eps_id_def} and Lemma \ref{lem:id_deep_gelu_approx} that
    \begin{align}
        \label{eq:phi_rec_phi_id_cfg}
        \notag
        &L_{id, \varphi} \lesssim \log(m\log(1 / \eps_{rec})), \quad
        \|W_{id, \varphi}\|_\infty \lesssim 1, \\
        &S_{id, \varphi} \lesssim \log(m\log(1 / \eps_{rec})),
        \quad \log B_{id, \varphi} \lesssim m^2 N + m^2 \log m + \log(1 / \eps_{rec}) \log m.
    \end{align}    
    Therefore, from Lemma \ref{lem:concat_nn} we find that $\varphi_{i} = \varphi_{rec, i} \circ \varphi_{id, i} \in \NN(L_{rec}, W_{rec}, S_{rec}, B_{rec})$, where the parameters of the neural network class are presented in \eqref{eq:phi_rec_i_cfg}.
    Next, we deduce from \eqref{eq:eps_id_def}, Lemma \ref{lem:id_deep_gelu_approx} and Lemma \ref{lem:div_approx_naive} that
    \begin{align}
        \label{eq:phi_i_R_bound}
        \notag
        \|\varphi_{rec, i} \circ \varphi_{id, i}\|_{W^{m, \infty}(\R)}
        &\leq \exp\{\cO(m \log (m + \|\varphi_{id, i}\|_{W^{m, \infty}(\R)}) )\} \|\varphi_{rec, i}\|_{W^{m, \infty}(\R)} \\
        &\leq \exp\{\cO(m^2 N + m^2\log(mN\log(1 / \eps_{rec})))\} .
    \end{align}
    Following this, consider the approximation properties of $\psi_{id, i} \circ \psi_{pou, i}$.
    From Lemma \ref{lem:pou_gelu_approx} and Lemma \ref{lem:comp_sob_norm} we find that
    \begin{align*}
        \max_{1 \leq i \leq N}\|\psi_{id, i} \circ \psi_{pou, i}\|_{W^{m, \infty}([a_0, 1] \setminus [a_{i - 2}, a_{i + 1} \wedge 1])}
        &\leq \exp\{\cO(m\log m)\}\|\psi_{id, i}\|_{W^{m, \infty}([-\eps_{pou}, \eps_{pou}])}(1 \vee \eps_{pou}^m) \\
        &\leq \exp\{\cO(m\log m)\}(\eps_{pou} + \eps_{id, \psi})
    \end{align*}
    and also
    \begin{align*}
        \|\psi_{id, i} \circ \psi_{pou, i} - \psi_{pou, i} \|_{W^{m, \infty}(\R)}
        &\leq \exp\{\cO(m \log m)\} \eps_{id, \psi} (1 \vee \|\psi_{pou, i}\|_{W^{m, \infty}(\R)}^m) \\
        &\leq \exp\{\cO(m^2N + m^2\log(m\log(1 / \eps_{pou})))\}\eps_{id, \psi} .
    \end{align*}
    Therefore, setting
    \begin{align}
        \label{eq:phi_rec_eps_id_psi_def}
        \log(1 / \eps_{id, \psi}) \asymp m^2\log(1 / \eps_{pou}) + m^2 N + m^2 \log m
    \end{align}
    ensures that
    \begin{align}
        \label{eq:psi_i_psi_pou_acc}
        \max_{1 \leq i \leq N}\|\psi_{id, i} \circ \psi_{pou, i} - \psi_{pou, i} \|_{W^{m, \infty}(\R)}
        \leq \eps_{pou}.
    \end{align}
    and
    \begin{align}
        \label{eq:psi_i_tail_bound}
        \max_{1 \leq i \leq N}\|\psi_{id, i} \circ \psi_{pou, i}\|_{W^{m, \infty}([a_0, 1] \setminus [a_{i - 2}, a_{i + 1} \wedge 1])}
        \leq \exp\{\cO(m\log m)\} \eps_{pou}.
    \end{align}
    From \eqref{eq:psi_i_psi_pou_acc} and Lemma \ref{lem:pou_gelu_approx} we also deduce that
    \begin{align}
        \label{eq:phi_rec_psi_i_bound_R}
        \max_{1 \leq i \leq N}\|\psi_i\|_{W^{m, \infty}(\R)}
        \leq \eps_{pou} + \max_{1 \leq i \leq N}\|\psi_{pou, i}\|_{W^{m, \infty}(\R)}
        \leq \exp\{\cO(mN + m\log(m\log(1 / \eps_{pou})))\}.
    \end{align}    
    Now derive the approximation error bound for $\varphi_{rec}$ defined in \eqref{eq:phi_rec_pou_def}.
    The triangle inequality suggests that
    \begin{align}
        \label{eq:phi_rec_acc_two_terms}
        \notag
        &\|\varphi_{rec} - f_{rec}\|_{W^{m, \infty}([a_0, 1])}
        \leq \underbrace{ \left\|f_{rec}\left(1 - \sum_{i = 1}^N\psi_i\right)\right\|_{W^{m, \infty}([a_0, 1])} }_{(A)} \\
        &\quad  + \underbrace{ \left\|\sum_{i=1}^N (f_{rec} - \varphi_i) \psi_i\right\|_{W^{m, \infty}([a_0, 1])} }_{(B)}
        + \underbrace{ \left\|\sum_{i=1}^N \varphi_i \cdot \psi_i - q(\varphi_i, \psi_i)\right\|_{W^{m, \infty}([a_0, 1])} }_{(C)} .
    \end{align}
    
    \noindent
    \textbf{Step 2: bounding term $(A)$.}\quad 
    From Lemma \ref{lem:prod_sob_norm} we deduce that
    \begin{align*}
        \left\|f_{rec}\left(1 - \sum_{i = 1}^N\psi_i\right)\right\|_{W^{m, \infty}([a_0, 1])}
        \leq 2^m \|f_{rec}\|_{W^{m, \infty}([a_0, 1])} \left\|1 - \sum_{i=1}^N\psi_i\right\|_{W^{m, \infty}([a_0, 1])}.
    \end{align*}
    Since $\sum_{i=1}^N \psi_{pou, i}(x) = 1$ for all $x \in \R$ due to Lemma \ref{lem:pou_gelu_approx} and $\|f_{rec}\|_{W^{m, \infty}([a_0, 1])} \leq \exp\{\cO(m\log(m / a_0))\}$, we obtain from \eqref{eq:psi_i_psi_pou_acc} that
    \begin{align}
        \label{eq:phi_rec_A_term_aux}
        \notag
        \left\|f_{rec}\left(1 - \sum_{i = 1}^N\psi_i\right)\right\|_{W^{m, \infty}([a_0, 1])}
        &\leq \exp\{\cO(m\log(m / a_0))\} N \eps_{pou} \\
        &\leq \exp\{\cO(mN + m\log m)\} \eps_{pou}.
    \end{align}

    \noindent
    \textbf{Step 3: bounding term $(B)$.}\quad
    The triangle inequality suggests that
    \begin{align}
        \label{eq:f_rec_min_sum_phi_psi}
        \left\|\sum_{i=1}^N (f_{rec} - \varphi_i) \psi_i\right\|_{W^{m, \infty}([a_0, 1])}
        \leq N\max_{1 \leq i \leq N} \|(f_{rec} - \varphi_i)\psi_i\|_{W^{m, \infty}([a_0, 1])} .
    \end{align}
    Moreover, for each $1 \leq i \leq N$ we have that
    \begin{align}
        \label{eq:f_rec_min_phi_psi_a_0_one}
        \notag
        &\|(f_{rec} - \varphi_i) \psi_i\|_{W^{m, \infty}([a_0, 1])} \\
        &\quad = \|(f_{rec} - \varphi_i) \psi_i\|_{W^{m, \infty}([a_{i-2}, a_{i+1}\wedge 1])} \vee \|(f_{rec} - \varphi_i) \psi_i\|_{W^{m, \infty}([a_0, 1] \setminus [a_{i-2}, a_{i+1} \wedge 1])}.
    \end{align}
    Now we analyze each term separately.
    As for the first term, we obtain from \eqref{eq:phi_i_f_rec_acc}, \eqref{eq:phi_rec_psi_i_bound_R} and Lemma \ref{lem:prod_sob_norm} that
    \begin{align}
        \label{eq:f_rec_min_phi_psi_i_aux}
        \notag
        \max_{1 \leq i \leq N}\|(f_{rec} - \varphi_i) \psi_i\|_{W^{m, \infty}([a_{i-2}, a_{i+1}\wedge 1])}
        &\leq 2^m \eps_{rec} \|\psi_i\|_{W^{m, \infty}(\R)} \\
        &\leq \exp\{\cO(mN + m\log(m\log(1 / \eps_{pou})))\} \eps_{rec} .
    \end{align}
    As for the second term, \eqref{eq:phi_i_R_bound} together with \eqref{eq:psi_i_tail_bound}, the fact that $\|f_{rec}\|_{W^{m, \infty}([a_0, 1])} \leq \exp\{\cO(m\log m + m N)\}$ and Lemma \ref{lem:prod_sob_norm} yield that for all $1 \leq i \leq N$
    \begin{align*}
        \|(f_{rec} - \varphi_i) \psi_i\|_{W^{m, \infty}([a_0, 1] \setminus [a_{i-2}, a_{i+1} \wedge 1])}
        &\leq \exp\{\cO(m\log m)\} (\|\varphi_i\|_{W^{m, \infty}(\R)} + \|f_{rec}\|_{W^{m, \infty}([a_0, 1])}) \eps_{pou} \\
        &\leq \exp\{\cO(m^2 N + m^2\log(mN\log(1 / \eps_{rec})))\} \eps_{pou} .
    \end{align*}
    Therefore, setting
    \begin{align}
        \label{eq:phi_rec_eps_pou_def}
        \log(1 / \eps_{pou}) \asymp \log(1 / \eps_{rec}) + m^2 N + m^2\log(mN \log(1 / \eps_{rec}))
    \end{align}
    guarantees
    \begin{align}
        \label{eq:f_rec_min_phi_i_psi_out}
        \max_{1 \leq i \leq N} \|(f_{rec} - \varphi_i) \psi_i\|_{W^{m, \infty}([a_0, 1] \setminus [a_{i-2}, a_{i+1} \wedge 1])}
        \leq \frac{\eps_{rec}}{3N}.
    \end{align}    
    Moreover, from \eqref{eq:f_rec_min_phi_psi_i_aux} we deduce that
    \begin{align*}
        \max_{1 \leq i \leq N}\|(f_{rec} - \varphi_i) \psi_i\|_{W^{m, \infty}([a_{i-2}, a_{i+1}\wedge 1])}
        \leq \exp\{\cO(mN + m\log(m\log(1 / \eps_{rec})))\} \eps_{rec} .
    \end{align*}
    Thus, setting
    \begin{align}
        \label{eq:phi_rec_eps_rec_def}
        \log(1 / \eps_{rec}) \asymp mN + m^2\log(1 / \eps)
    \end{align}
    ensures that
    \begin{align}
        \label{eq:f_rec_min_phi_i_psi_in}
        \max_{1 \leq i \leq N}\|(f_{rec} - \varphi_i) \psi_i\|_{W^{m, \infty}([a_{i-2}, a_{i+1}\wedge 1])}
        \leq \frac{\eps}{3 N}.
    \end{align}
    The combination of \eqref{eq:f_rec_min_phi_psi_a_0_one}, \eqref{eq:f_rec_min_phi_i_psi_out} and \eqref{eq:f_rec_min_phi_i_psi_in} imply that
    \begin{align*}
        \max_{1 \leq i \leq N} \|(f_{rec} - \varphi_i) \psi_i\|_{W^{m, \infty}([a_0, 1])}
        \leq \frac{\eps}{3N}.
    \end{align*}
    Therefore, we deduce from \eqref{eq:f_rec_min_sum_phi_psi} that
    \begin{align}
        \label{eq:f_rec_min_exact_prod_acc}
        \left\|\sum_{i=1}^N (f_{rec} - \varphi_i) \psi_i\right\|_{W^{m, \infty}([a_0, 1])}
        \leq \eps / 3.
    \end{align}
    
    \noindent
    \textbf{Step 4: bounding term $(C)$.}\quad
    We first note from \eqref{eq:phi_i_R_bound} and \eqref{eq:phi_rec_eps_rec_def} that
    \begin{align}
        \label{eq:phi_rec_phi_i_R}
        \max_{1 \leq i \leq N}\|\varphi_i\|_{W^{m, \infty}(\R)}
        \leq \exp\{\cO(m^2 N + m^2\log(m\log(1 / \eps)))\} .
    \end{align}
    In addition, from \eqref{eq:phi_rec_psi_i_bound_R}, \eqref{eq:phi_rec_eps_pou_def} and \eqref{eq:phi_rec_eps_rec_def} we find that
    \begin{align}
        \label{eq:phi_rec_psi_i_R}
        \notag
        \max_{1 \leq i \leq N}\|\psi_i\|_{W^{m, \infty}(\R)}
        &\leq \exp\{\cO(mN + m\log(m\log(1 / \eps_{pou})))\} \\
        &\leq \exp\{\cO(mN + m\log(m\log(1 / \eps)))\}.
    \end{align}
    These observations together with Corollary \ref{co:multi_approx_gelu} and Lemma \ref{lem:comp_sob_norm} imply that
    \begin{align*}
        \left\|\sum_{i=1}^N \varphi_i \cdot \psi_i - q(\varphi_i, \psi_i)\right\|_{W^{m, \infty}([a_0, 1])}
        &\leq \sum_{i=1}^N \exp\{\cO(m \log (m + \|\varphi_i\|_{W^{m, \infty}(\R)} + \|\psi_i\|_{W^{m, \infty}(\R)}) )\}\eps_{mul} \\
        &\leq \exp\{\cO(m^3 N + m^3 \log(m\log(1 / \eps)))\} \eps_{mul} .
    \end{align*}
    Thus, setting
    \begin{align}
        \label{eq:phi_rec_eps_mul_def}
        \log(1 / \eps_{mul}) \asymp m^3 N + m^3 \log(m / \eps)
    \end{align}
    guarantees that
    \begin{align}
        \label{eq:phi_rec_C_term_acc}
        \left\|\sum_{i=1}^N \varphi_i \cdot \psi_i - q(\varphi_i, \psi_i)\right\|_{W^{m, \infty}([a_0, 1])}
        \leq \eps / 3.
    \end{align}    
    
    \noindent
    \textbf{Step 5: combining $(A)$, $(B)$ and $(C)$ together.}\quad
    From \eqref{eq:phi_rec_A_term_aux}, \eqref{eq:phi_rec_eps_pou_def} and \eqref{eq:phi_rec_eps_rec_def} we deduce that the term $(A)$ is evaluated as
    \begin{align*}
        \left\|f_{rec}\left(1 - \sum_{i = 1}^N\psi_i\right)\right\|_{W^{m, \infty}([a_0, 1])}
        \leq \eps / 3.
    \end{align*}    
    Therefore, combining this bound with \eqref{eq:phi_rec_acc_two_terms}, \eqref{eq:f_rec_min_exact_prod_acc} and \eqref{eq:phi_rec_C_term_acc} yields
    \begin{align*}
        \|\varphi_{rec} - f_{rec}\|_{W^{m, \infty}([a_0, 1])} \leq \eps.
    \end{align*}
    In addition, from Lemma \ref{lem:comp_sob_norm} we deduce that
    \begin{align*}
        \|\varphi_{rec}\|_{W^{m, \infty}(\R)}
        &\leq N \max_{1 \leq i \leq N}\|q(\phi_i, \psi_i)\|_{W^{m, \infty}(\R)} \\
        &\leq N \exp\{\cO(m\log m)\}\max_{1 \leq i \leq N}\|q\|_{W^{m, \infty}(\varphi_i(\R) \times \psi_i(\R))}(1 \vee \|\varphi_i\|_{W^{m, \infty}(\R)}^m \vee \|\psi_i\|_{W^{m, \infty}(\R)}^m) .
    \end{align*}    
    Corollary \ref{co:multi_approx_gelu} together with \eqref{eq:phi_rec_phi_i_R} and \eqref{eq:phi_rec_psi_i_R}
    \begin{align*}
        \|\varphi_{rec}\|_{W^{m, \infty}(\R)}
        &\leq \max_{1 \leq i \leq N} N \exp\{\cO( m\log(m + \|\varphi_i\|_{W^{m, \infty}(\R)} + \|\psi_i\|_{W^{m, \infty}(\R)}) )\} \\
        &\leq \exp\{\cO(m^3 N + m^3\log(m\log(1 / \eps)))\} .
    \end{align*}
    
    \noindent
    \textbf{Step 6: deriving the configuration of $\varphi_{rec}$.}\quad
    First, from \eqref{eq:phi_rec_i_cfg} and \eqref{eq:phi_rec_eps_rec_def} it follows that for each $i \in \{1, \dots, N\}$, we have $\varphi_i \in \NN(L_{rec}, W_{rec}, S_{rec}, B_{rec})$ with
    \begin{align*}
        &L_{rec} \lesssim \log(mN) + \log\log(1 / \eps),
        \quad \|W_{rec}\|_\infty \vee S_{rec} \lesssim m^8(N^4 + m^4\log^4(1 / \eps)), \\
        &\log B_{rec} \lesssim m^8(N^4 + m^4\log^4(1 / \eps)) .
    \end{align*}    
    Second, from \eqref{eq:phi_rec_eps_id_psi_def}, \eqref{eq:phi_rec_eps_pou_def} and \eqref{eq:phi_rec_eps_rec_def} we deduce that 
    \begin{align}
        \label{eq:phi_rec_eps_aux_eps}
        \log(1 / \eps_{pou}) \lesssim m^2(N + \log(m / \eps)), \quad
        \log(1 / \eps_{id, \psi}) \lesssim m^4(N + \log(m / \eps)).
    \end{align}
    Hence, \eqref{eq:phi_rec_phi_id_cfg}, \eqref{eq:phi_rec_eps_rec_def} and Lemma \ref{lem:id_deep_gelu_approx} imply that for each $1 \leq i \leq N$ we have $\psi_{id, i} \in \NN(L_{id, \psi}, W_{id, \psi}, S_{id, \psi}, B_{id, \psi})$ with
    \begin{align*}
        &L_{id, \psi} \vee S_{id, \psi} \lesssim L_{id, \varphi} \lesssim \log(mN\log(1 / \eps)),
        \quad \|W_{id, \psi}\|_\infty \lesssim 1, \\
        &\log B_{id, \psi} \lesssim (m + L_{id, \psi})\log m + \log(1 / \eps_{id, \psi}) + m\log\left(\max_{1 \leq i \leq N}\|\psi_{pou, i}\|_{W^{m, \infty}(\R)} \right).
    \end{align*}
    Form \eqref{eq:phi_rec_eps_aux_eps} and Lemma \ref{lem:pou_gelu_approx} it follows that
    \begin{align*}
        \log B_{id, \psi} \lesssim m^4(N + \log(m / \eps)) .
    \end{align*}    
    Lemma \ref{lem:pou_gelu_approx} together with \eqref{eq:phi_rec_eps_aux_eps} imply that for all $1 \leq i \leq N$ it holds that $\psi_{pou, i} \in \NN(L_{pou}, W_{pou}, S_{pou}, B_{pou})$ with
    \begin{align*}
        L_{pou} \vee \|W_{pou}\| \vee S_{pou} \lesssim 1,
        \quad \log B_{pou} \lesssim \log(1 / \eps_{pou}) + mN + m\log m
        \lesssim m^2(N + \log(m / \eps)) .
    \end{align*}
    Therefore, Lemma \ref{lem:concat_nn} yields that $\psi_i = \psi_{id, i} \circ \psi_{pou, i} \in \NN(L_{id, \psi}, W_{id, \psi}, S_{id, \psi}, B_{id, \psi})$.
    In addition, from \eqref{eq:phi_rec_eps_mul_def} and Corollary \ref{co:multi_approx_gelu} we find that $q \in \NN(L_{mul}, W_{mul}, S_{mul}, B_{mul})$ with
    \begin{align*}
        L_{mul} \vee \|W_{mul}\| \vee S_{mul} \lesssim 1,
        \quad \log B_{mul} \lesssim \log m + \log(1 / \eps_{mul})
        \lesssim m^3(N + \log(m / \eps)) .
    \end{align*}
    Thus, applying Lemma \ref{lem:concat_nn} and Lemma \ref{lem:paral_nn}, we obtain that
    \begin{align*}
        &L \lesssim L_{mul} + L_{id, \psi}
        \lesssim \log(mN\log(1 / \eps)),
        \quad \|W\|_\infty \vee S \lesssim N \|W_{rec}\|_{\infty}
        \lesssim m^8 N(N^4 + m^4\log^4(1 / \eps)), \\
        &\log B \lesssim \log B_{rec} + \log B_{id, \psi} + \log \|W_{rec}\|_\infty + \log N
        \lesssim  m^8(N^4 + m^4\log^4(1 / \eps)).
    \end{align*}
    The proof is complete.
     
\end{proof}

By comparing our Lemma \ref{lem:recip_gelu_approx} to \cite[Lemma A.4]{yakovlev2025generalization}, we see that the parameter count for high-order Sobolev approximation is $\cO(N^5 + N \log^4(1 / \eps))$, slightly exceeding their bound of $\cO(N^4 + N \log^3(1 / \eps))$.
Nevertheless, we generalize the approximation capabilities to high-order Sobolev norms.

Finally, we present a result on the division approximation, utilizing the reciprocal-based approach we have developed.

\begin{Lem}[division operation approximation]
    \label{lem:div_gelu_approx}
    Define $\mathrm{div} : (x, y) \mapsto x / y$ for any $x \in \R$ and $y > 0$.
    Let also $a_0 = 2^{-N}$ for $N \in \N$ with $N \geq 3$.
    Then, for every $\eps \in (0, 1)$ and every $m \in \N$ such that $m \geq 3$, there exists a GELU network $\varphi_{div} \in \NN(L, W, S, B)$ satisfying
    \begin{align*}
        (i) &\quad \|\varphi_{div} - \mathrm{div}\|_{W^{m, \infty}([-1, 1] \times [a_0, 1])} \leq \eps, \\
        (ii) &\quad \| \varphi_{div} \|_{W^{m, \infty}(\R^2)} \leq \exp\{\cO(m^4 N + m^4\log(m\log(1 / \eps)))\} .
    \end{align*}
    Furthermore, the network $\varphi_{div}$ has
    \begin{align*}
        L \lesssim \log(mN\log(1 / \eps)),
        \quad \|W\|_\infty \vee S \lesssim m^{21} N^5 \log^4(1 / \eps),
        \quad \log B \lesssim m^{24} N^4 \log^4(1 / \eps).
    \end{align*}
    
\end{Lem}

The proof of Lemma \ref{lem:div_gelu_approx} can be found in Appendix \ref{sec:lem_div_gelu_approx_proof}.

\bibliographystyle{abbrvnat}
\bibliography{references}

\appendix

\section{Deferred proofs}

\subsection{Proof of Lemma \ref{lem:id_deep_gelu_approx}}
\label{sec:lem_id_deep_gelu_approx_proof}
We first prove the statement for $K = 1$ and then generalize to $K \geq 1$.
Let
\begin{align*}
    \varphi_{j} = \varphi_{id, j} \circ \varphi_{j-1}, \quad 2 \leq j \leq L,
\end{align*}
where $\varphi_1 = \id$ and $\varphi_{id, j}$ is a GELU network from Lemma \ref{lem:id_gelu_approx} that approximates the identity operation with the accuracy parameter $\eps_{id}^{(j)}$.
Formally, for each $2 \leq j \leq L$ we have
\begin{align}
    \label{eq:phi_id_j_acc}
    \|\varphi_{id, j} - \id\|_{W^{m, \infty}([-C, C])} \leq C^2 \eps_{id}^{(j)} \quad \text{for all } C \geq 1.
\end{align}
For every $1 \leq j \leq L$ we introduce $\eps_j = \|\varphi_j - \id\|_{W^{m, \infty}([-1, 1])}$.
Therefore, the triangle inequality for every $2 \leq j \leq L$ implies that
\begin{align*}
    \|\varphi_{id, j} \circ \varphi_{j - 1} - \id\|_{W^{m, \infty}([-1, 1])}
    \leq \| \varphi_{j - 1} - \id \|_{W^{m, \infty}([-1, 1])}
    + \|(\id -  \varphi_{id, j}) \circ \varphi_{j - 1}\|_{W^{m, \infty}([-1, 1])}.
\end{align*}
Next, applying Lemma \ref{lem:comp_sob_norm}, we obtain that
\begin{align*}
    \eps_j
    \leq \eps_{j - 1} + 16(e^2 m^4)^m \|\varphi_{id, j} - \id\|_{W^{m, \infty}([-1 - \eps_{j - 1}, 1 + \eps_{j -1}])} (1 \vee \|\varphi_{j-1}\|^m_{W^{m, \infty}([-1, 1])}),
\end{align*}
Therefore, \eqref{eq:phi_id_j_acc} suggests that
\begin{align*}
    \eps_j
    \leq \eps_{j - 1} + 16(e^2 m^4)^m (1 + \eps_{j - 1})^{m + 2} \eps_{id}^{(j)}.
\end{align*}
Now choosing
\begin{align}
    \label{eq:eps_id_j}
    \eps_{id}^{(j)} = 16 (e^2 m^4)^{-m} \eps_2 \in (0, 1), \quad 2 \leq j \leq L ,
\end{align}
we find that
\begin{align}
    \label{eq:eps_id_recursion}
    \eps_j
    \leq \eps_{j - 1} + (1 + \eps_{j - 1})^{m + 2}\eps_{j - 1},
    \quad 2 \leq j \leq L.
\end{align}
Suppose that for each $2 \leq j \leq L$, the approximation error is given by $\eps_j = 2^{\gamma_j}\eps_2$ with $\gamma_2 = 0$.
We also set a helper $\gamma_1 = 0$.
Hence, considering \eqref{eq:eps_id_recursion}, we conclude that
\begin{align*}
    2^{\gamma_j}\eps_2
    \leq 2^{\gamma_{j - 1}}\eps_2 + 2^{(m + 2)(\gamma_{j - 1} + 1)}\eps_2
    \leq 2^{(m + 3)\gamma_{j - 1} + m + 3}\eps_2.
\end{align*}
Therefore, $\gamma_j \leq (2(m + 3))^j$ for each $2 \leq j \leq L$.
Setting $\eps_{id}^{(2)} = \eps' (2(m + 3))^{-L}$ for some $\eps' \in (0, 1)$, we deduce from \eqref{eq:eps_id_j} that for any $2 \leq j \leq L$
\begin{align*}
    \log(1 / \eps_{id}^{(j)})
    \lesssim m\log m + \log(1 / \eps_{mul}^{(2)})
    \lesssim (m + L)\log m + \log(1 / \eps').
\end{align*}
Moreover,
\begin{align*}
    \|\varphi_L - \id\|_{W^{m, \infty}([-1, 1])} \leq \eps'.
\end{align*}
Next, using Lemma \ref{lem:concat_nn} and Lemma \ref{lem:id_gelu_approx} we find that $\varphi_L \in \NN(L, W_{id}, S_{id}, B_{id})$ with
\begin{align}
    \label{eq:phi_L_cfg}
    \|W_{id}\|_\infty \lesssim 1,
    \quad S_{id} \lesssim L,
    \quad \log B_{id} \lesssim \log m + \max_{2 \leq j \leq L}\log(1 / \eps_{id}^{(j)})
    \lesssim (m + L)\log m + \log(1 / \eps').
\end{align}
The generalization to the case when $K \geq 1$ is trivial.
Let $\varphi_{L, K}(x) = K\varphi_L(x / K)$ for any $x \in \R$.
Then we have that
\begin{align}
    \label{eq:phi_id_L_K_acc}
    \|\varphi_{L, K} - \id \|_{W^{m, \infty}([-K, K])}
    \leq K \|\varphi_{L} - \id\|_{W^{m, \infty}([-1, 1])}
    \leq K\eps'.
\end{align}
As a final step, we add a clipping operation to ensure that the resulting function has finite norm on a real line.
Let $\varphi_{clip}$ be a clipping operation approximation from Lemma \ref{lem:clip_gelu_approx} with the accuracy parameter $\eps'$ and the scale parameter $K$.
Thus, we have that
\begin{align}
    \label{eq:id_phi_clip_bounds}
    \notag
    &\|\varphi_{clip} - \id\|_{W^{m, \infty}([-K, K])} \leq \eps',
    \quad \|\varphi_{clip}\|_{W^{0, \infty}(\R)} \leq 4K, \\
    &\|\varphi_{clip}\|_{W^{m, \infty}(\R)} \leq \exp\{\cO(m \log(m\log(1 / \eps')) + \log(2K))\}.
\end{align}
In addition, $\varphi_{clip} \in \NN(L_{clip}, W_{clip}, S_{clip}, B_{clip})$ with
\begin{align}
    \label{eq:phi_id_clip_cfg}
    L_{clip} \vee \|W_{clip}\|_\infty \vee S_{clip} \lesssim 1,
    \quad \log B_{clip} \lesssim \log (K m / \eps').
\end{align}
Then it holds due to the triangle inequality and \eqref{eq:phi_id_L_K_acc} that
\begin{align}
    \label{eq:phi_L_4K_clip_acc_aux}
    \|\varphi_{L, 4K} \circ \varphi_{clip} - \id\|_{W^{m, \infty}([-K, K])}
    \leq 4K\eps' + \|\varphi_{L, 4K} \circ \varphi_{clip} - \varphi_{L, 4K}\|_{W^{m, \infty}([-K, K])}.
\end{align}
Hence, Lemma \ref{lem:comp_sob_norm} together with \eqref{eq:id_phi_clip_bounds} imply that
\begin{align*}
    \|\varphi_{L, 4K} \circ \varphi_{clip} - \varphi_{L, 4K}\|_{W^{m, \infty}([-K, K])}
    \leq \exp\{\cO(m\log m)\} \|\varphi_{L, 4K}\|_{W^{m + 1, \infty}([-4K, 4K])}\eps' (\eps' + K)^{2m}.
\end{align*}
Next we note that \eqref{eq:phi_L_cfg} and \eqref{eq:phi_id_L_K_acc} are true if the smoothness parameter is $m + 1$ instead of $m$.
Then we have from \eqref{eq:phi_id_L_K_acc} that
\begin{align*}
    \|\varphi_{L, 4K} \circ \varphi_{clip} - \varphi_{L, 4K}\|_{W^{m, \infty}([-K, K])}
    \leq \exp\{ \cO(m \log(mK) ) \} \eps'.
\end{align*}
Therefore, \eqref{eq:phi_L_4K_clip_acc_aux} is evaluated as
\begin{align*}
    \|\varphi_{L, 4K} \circ \varphi_{clip} - \id\|_{W^{m, \infty}([-K, K])}
    \leq \exp\{ \cO(m \log(mK) ) \} \eps' .
\end{align*}
Thus, setting
\begin{align}
    \label{eq:phi_id_eps_prime_def}
    \log(1 / \eps') \asymp \log(1 / \eps) + m\log(mK)
\end{align}
ensures that
\begin{align*}
    \|\varphi_{L, 4K} \circ \varphi_{clip} - \id\|_{W^{m, \infty}([-K, K])}
    \leq \eps.
\end{align*}
In addition, from \eqref{eq:phi_id_L_K_acc}, \eqref{eq:id_phi_clip_bounds}, \eqref{eq:phi_id_eps_prime_def} and Lemma \ref{lem:comp_sob_norm} we obtain that
\begin{align*}
    \|\varphi_{L, 4K} \circ \varphi_{clip}\|_{W^{m, \infty}(\R)}
    &\leq \exp\{ \cO(m \log m) \} 4K(1 + \eps') \exp\{\cO(m \log(m\log(1 / \eps')) + \log(2K))\} \\
    &\leq \exp\{\cO(m \log(m\log(1 / \eps)) + \log(2K))\}.
\end{align*}
Finally, Lemma \ref{lem:concat_nn} combined with \eqref{eq:phi_L_cfg}, \eqref{eq:phi_id_clip_cfg} and \eqref{eq:phi_id_eps_prime_def} yields that $\varphi_{id} =  \varphi_{L, 4K} \circ \varphi_{clip}$ has
\begin{align*}
    \|W\|_\infty \lesssim 1, \quad S \lesssim L,
    \quad \log B \lesssim (m + L)\log m + \log(1 / \eps) + m\log(K).
\end{align*}
This finishes the proof.

\myendproof

\subsection{Proof of Corollary \ref{co:multi_approx_gelu}}
\label{sec:co_multi_approx_gelu_proof}

We are going to reduce the multiplication to the case of square operation by letting
\begin{align}
    \label{eq:phi_mul_def}
    \varphi_{mul}(x, y) := \frac{1}{4}\left(\varphi_{sq}(x + y) - \varphi_{sq}(x - y)\right),
\end{align}
where $\varphi_{sq}$ is a GELU network from Lemma \ref{lem:square_approx} with the accuracy parameter $\eps / 4$.
Therefore, using the observation that for any $\alpha = (\alpha_1, \alpha_2)^\top \in \Z_+^2$, it holds that
\begin{align*}
    D^{\alpha}\varphi_{mul}(x, y) = \frac{1}{4}(D^{|\alpha|}\varphi_{sq}(x + y) - (-1)^{\alpha_2}\varphi_{sq}(x - y)), \quad \text{for all } x, y \in \R,
\end{align*}
leads to
\begin{align*}
    \|\varphi_{mul} - \mathrm{prod}_2\|_{W^{m, \infty}([-C, C]^2)}
    \leq \frac{1}{2}\|\varphi_{sq} - f_{sq}\|_{W^{m, \infty}([-2C, 2C])}
    \leq C^3 \eps ,
\end{align*}
where $C \geq 1$ is arbitrary, and the last inequality uses Lemma \ref{lem:square_approx}.
Finally, in view of \eqref{eq:phi_mul_def}, we deduce that the summation argument outlined in Lemma \ref{lem:paral_nn} yields the configuration in the statement.
This completes the proof.

\myendproof

\subsection{Proof of Corollary \ref{co:monom_approx_gelu}}
\label{sec:co_monom_approx_gelu_proof}

We first introduce a flatten operation as follows:
\begin{align*}
    \mathrm{flat}_{\bk}(x_1, \dots, x_I) = (\underbrace{x_1, \dots, x_1}_{k_1 \text{ times}}, \dots, \underbrace{x_I, \dots, x_I}_{k_I \text{ times}} )^\top.
\end{align*}
Now let $\varphi_{mul, d}$ be a neural network from Lemma \ref{lem:mul_d_gelu_approx} with the accuracy parameter $\tilde \eps \in (0, 1)$, which will be specified a bit later in the proof, and scale parameter $K$.
Therefore, using Lemma \ref{lem:comp_sob_norm}, we derive an approximation accuracy for $\varphi_{mul, \bk} = \varphi_{mul, d} \circ \mathrm{flat}_\bk$:
\begin{align*}
    \|\varphi_{mul, \bk} - \mathrm{prod}_{\bk}\|_{W^{m, \infty}([-K, K]^I)}
    &= \|(\varphi_{mul, d} - \mathrm{prod}_d) \circ \mathrm{flat}_\bk\|_{W^{m, \infty}([-K, K]^I)} \\
    &\leq \exp\{\cO(m \log(m d))\} \|\varphi_{mul, d} - \mathrm{prod}_d\|_{W^{m, \infty}([-K, K]^d)} K^m.
\end{align*}
Next, the approximation property of $\varphi_{mul, d}$ implies that
\begin{align*}
        \|\varphi_{mul, \bk} - \mathrm{prod}_{\bk}\|_{W^{m, \infty}([-K, K]^I)}
        \leq  \exp\{\cO(m \log(mdK))\} \tilde\eps .
\end{align*}
To continue, we set
\begin{align}
    \label{eq:monom_eps_tilde}
    \log(1 / \tilde{\eps}) \asymp \log(1 / \eps) + m \log(mdK)
\end{align}
and arrive at
\begin{align*}
    \|\varphi_{mul, \bk} - \mathrm{prod}_{\bk}\|_{W^{m, \infty}([-K, K]^I)}
    \leq \eps.
\end{align*}
In addition, Lemma \ref{lem:mul_d_gelu_approx} together with Lemma \ref{lem:comp_sob_norm} suggest that
\begin{align*}
    \|\varphi_{mul, \bk}\|_{W^{m, \infty}(\R^I)}
    = \|\varphi_{mul, d} \circ \mathrm{flat}_\bk\|_{W^{m, \infty}(\R^I)}
    \leq \exp\{\cO((m^2 + d)\log(mdK \log(1 / \tilde{\eps})))\}.
\end{align*}
From \eqref{eq:monom_eps_tilde} we obtain that
\begin{align*}
    \|\varphi_{mul, \bk}\|_{W^{m, \infty}(\R^I)}
    \leq \exp\{\cO((m^2 + d)\log(mdK \log(1 / \eps)))\} .
\end{align*}
To finalize the proof, we formulate the configuration of $\varphi_{mul, \bk}$.
Note that $\mathrm{flat}_\bk$ is implemented using a single linear layer without a bias term, and its weight matrix contains binary values.
Consequently, the choice of $\tilde\eps$ given in \eqref{eq:monom_eps_tilde} combined with Lemma \ref{lem:mul_d_gelu_approx} and the concatenation result outlined in Lemma \ref{lem:concat_nn} ensures that
$\varphi_{mul, \bk}$ has $L \lesssim \log d$, $S \vee \|W\|_\infty \lesssim (d \vee I)^3$ and
\begin{align*}
    \log B
    \lesssim (\log(1 / \eps) + (d + m)\log K + m^2d^2) \log d + \log I.
\end{align*}
This completes the proof.

\myendproof

\subsection{Proof of Lemma \ref{lem:div_approx_naive}}
\label{sec:lem_div_approx_naive_proof}

For some $r \in \N$ with $r \geq m$, which will be optimized later, consider $f_r(x) = \frac{1}{b}\sum_{i=0}^{r - 1}(1 - x / b)^i$.
We next note that for all $x \in [a, b]$
\begin{align*}
    \frac{1}{x} - f_r(x)
    = \frac{1}{b}\sum_{i=0}^\infty (1 - x / b)^i - \frac{1}{b}\sum_{i=0}^{r - 1} (1 - x / b)^i
    = \frac{(1 - x / b)^r}{x}
\end{align*}
Then, by lemma \ref{lem:prod_sob_norm}, the approximation accuracy of $f_r$ is
\begin{align*}
    \|f_{rec} - f_r\|_{W^{m, \infty}([a, b])}
    &\leq 2^m \|f_{rec}\|_{W^{m, \infty}([a, b])} (b \wedge 1)^{-m} r^m (1 - a / b)^{r - m} \\
    &\leq \left(\frac{2r}{b \wedge 1}\right)^m a^{-m-1} (1 - a / b)^{r - m} m! \\
    &\leq \frac{m!}{a}\left(\frac{2r}{a(b \wedge 1)}\right)^m \exp\left(-\frac{(r - m)a}{b}\right),
\end{align*}
where the last inequality uses $1 + x \leq e^x$ for any $x \in \R$.
Therefore, setting $r = \ceil{m + \frac{b}{a}\log(1 / \eps')}$ for some $\eps' \in (0, 1)$ guarantees that
\begin{align*}
    \|f_{rec} - f_r\|_{W^{m, \infty}([a, b])}
    \leq \frac{m!}{a}\left(\frac{2r}{a(b \wedge 1)}\right)^m \eps'.
\end{align*}
We now set $\eps' = \frac{a}{4 m!}\left(\frac{a(b \wedge 1)}{2r}\right)^m \eps$ with $\eps \in (0, 1)$, which leads to
\begin{align}
    \label{eq:f_rec_f_r_acc}
    \|f_{rec} - f_r\|_{W^{m, \infty}([a, b])} \leq \eps / 4.
\end{align}
We also deduce from Stirling's approximation that
\begin{align*}
    \log(1 / \eps')
    &\lesssim \log(1 / \eps)+ m \log (m / ab) + m\log\left(m + \frac{b}{a}\log(1 / \eps')\right) \\
    &\lesssim \log(1 / \eps) + m\log(m / a) + m \log\log(1 / \eps').
\end{align*}
The last inequality suggests that $\log(1 / \eps') \lesssim \log(1 / \eps) + m\log(m / a)$, since the inequality $x \lesssim a + b\log x$ yields $x \lesssim a + b\log b$ for any positive $a$, $b$ and $x$.
Let $\varphi_{part}$ be a GELU network from Lemma \ref{lem:part_sum_monomial} with the accuracy parameter $\eps_{part}$, the scale parameter $K = 1 + 1/b$, the parameter $I = 1$ and $d = r$.
Then, for $f_{part}(x) = \sum_{i=0}^{r - 1}x^i$ and $\tilde{\varphi}_{rec} = (1 / b)\varphi_{part} \circ (1 - \id / b)$ it holds that
\begin{align*}
    \|\tilde{\varphi}_{rec} - f_r\|_{W^{m, \infty}([a, b])}
    \leq b^{-1} \| (\varphi_{part} - f_{part}) \circ (1 - \id / b) \|_{W^{m, \infty}([a, b])}
    \leq (b \wedge 1)^{-m-1} \eps_{part},
\end{align*}
where the last inequality follows from the chain rule.
Thus, setting
\begin{align}
    \label{eq:phi_rec_eps_part_def}
    \log(1 / \eps_{part}) \asymp \log(1 / \eps) + m \log(1 / a),
\end{align}
we obtain from \eqref{eq:f_rec_f_r_acc} that
\begin{align}
    \label{eq:phi_rec_naive_aux}
    \|\tilde{\varphi}_{rec} - f_{rec}\|_{W^{m, \infty}([a, b])} \leq \eps / 2 .
\end{align}
In addition, Lemma \ref{lem:part_sum_monomial} together with Lemma \ref{lem:concat_nn} imply that
$\tilde{\varphi}_{rec} \in \NN(L_{rec}, W_{rec}, S_{rec}, B_{rec})$ with
\begin{align}
    \label{eq:phi_rec_naive_cfg_aux}
    \notag
    L_{rec} &\lesssim \log r \lesssim \log((mb / a)\log(1 / \eps)),
    \quad \|W_{rec}\|_\infty \vee S_{rec} \lesssim r^4 \lesssim (mb/a)^4\log^4(m/\eps a), \\
    \notag
    \log B_{rec} &\lesssim (\log(1 / \eps_{part}) + m^2 r \log(mr) + m^2 r^2) \log r \\
    &\lesssim (m^4 b^2 / a^2)\log^2(1 / \eps a)\log^2( (mb / a)\log(1 / \eps a)).
\end{align}
We further observe that through rescaling of the parameters $a$ and $b$ and increasing the smoothness parameter $m$, the bound \eqref{eq:phi_rec_naive_aux} remains valid for $a / 2$ and $2b$, while preserving the configuration specified \eqref{eq:phi_rec_naive_cfg_aux} remains the same.
Fromally, we have that
\begin{align}
    \label{eq:tilde_phi_rec_naive_acc}
    \|\tilde{\varphi}_{rec} - f_{rec}\|_{W^{m + 1, \infty}([a / 2, 2b])} \leq \eps / 2.
\end{align}
Furthermore, the derived bound together with the fact that $\|f_{rec}\|_{W^{m, \infty}([a/2, 2b])} \leq \exp\{\cO(m\log m + m N)\}$ imply that
\begin{align}
    \label{eq:rec_naive_tilde_rec_R}
    \|\tilde{\varphi}_{rec}\|_{W^{m, \infty}([a/2, 2b])} \leq \exp\{\cO(m\log m + m N)\} .
\end{align}
Now let $\varphi_{clip}$ be the clipping operation approximation from Lemma \ref{lem:clip_gelu_approx} with the precision parameter $\eps_{clip}$ and the scale parameter $(4b/a - 4) \geq 1$.
Let also
\begin{align*}
    \breve{\varphi}_{clip}(x) = \varphi_{clip}(x - 4 - 4b/a) + 4b/a + 4, \quad x \in \R .
\end{align*}
Therefore, from Lemma \ref{lem:clip_gelu_approx} we find that $\breve{\varphi}_{clip}$ satisfies
\begin{align*}
    (i) &\quad \| \breve{\varphi}_{clip} - \id \|_{W^{m, \infty}([8, 8b/a])}
    \leq \|\varphi_{clip} - \id\|_{W^{m, \infty}([-(4b/a - 4), 4b/a - 4])}
    \leq \eps_{clip}, \\
    (ii) &\quad 11 / 2 \leq \breve{\varphi}_{clip}(x) \leq 8b/a + 5/2, \quad \text{for all } x \in \R, \\
    (iii) &\quad \|\breve{\varphi}_{clip}\|_{W^{m, \infty}(\R)} \leq \exp\{\cO(m\log m + m\log\log(1 / \eps_{clip}) + \log(b/a))\} .
\end{align*}
Moreover, for $\tilde{\varphi}_{clip}(x) = (a/8)\breve{\varphi}_{clip}(8x/a)$ we deduce from the chain rule and property $(i)$ that
\begin{align}
    \label{eq:rec_naive_tilde_clip_acc}
    \notag
    \| \tilde{\varphi}_{clip} - \id \|_{W^{m, \infty}([a, b])}
    &\leq \exp\{\cO(m\log(1 / a))\} \| \breve{\varphi}_{clip} - \id \|_{W^{m, \infty}([8, 8b/a])} \\
    &\leq \exp\{\cO(m\log(1 / a))\}\eps_{clip}.
\end{align}
From property $(ii)$ it follows that
\begin{align}
    \label{eq:rec_naive_tilde_clip_range}
    a/2 \leq \tilde{\varphi}_{clip}(x)
    \leq 2b, \quad \text{for all } x \in \R .
\end{align}
In addition, property $(iii)$ yields
\begin{align}
    \label{eq:rec_naive_tilde_clip_R}
    \|\tilde{\varphi}_{clip}\|_{W^{m, \infty}(\R)} \leq \exp\{\cO(m\log(m / a) + m\log\log(1 / \eps_{clip}))\} .
\end{align}
Combining \eqref{eq:rec_naive_tilde_clip_acc}, \eqref{eq:rec_naive_tilde_clip_range} and Lemma \ref{lem:comp_sob_norm}, we obtain for $\varphi_{rec} = \tilde{\varphi}_{rec} \circ \tilde{\varphi}_{clip}$ that
\begin{align*}
    \| \tilde{\varphi}_{rec} \circ \tilde{\varphi}_{clip} - \tilde{\varphi}_{rec} \|_{W^{m, \infty}([a, b])} 
    \leq \exp\{\cO(m \log (m / a) )\} \| \tilde{\varphi}_{rec }\|_{W^{m + 1, \infty}([a/2, 2b])} \eps_{clip} .
\end{align*}
From \eqref{eq:rec_naive_tilde_rec_R} we find that
\begin{align*}
    \| \tilde{\varphi}_{rec} \circ \tilde{\varphi}_{clip} - \tilde{\varphi}_{rec} \|_{W^{m, \infty}([a, b])} 
    \leq \exp\{\cO(m N + m \log (m / a) )\}  \eps_{clip} .
\end{align*}
Choosing
\begin{align}
    \label{eq:rec_naive_eps_clip}
    \log(1 / \eps_{clip}) \asymp \log(1 / \eps) + m N + m \log (m / a) 
\end{align}
and combining the derived bound with \eqref{eq:tilde_phi_rec_naive_acc}, it follows that
\begin{align*}
    \|\varphi_{rec} - f_{rec}\|_{W^{m, \infty}([a, b])} \leq \eps.
\end{align*}
We also find from \eqref{eq:rec_naive_tilde_rec_R}, \eqref{eq:rec_naive_tilde_clip_range}, \eqref{eq:rec_naive_tilde_clip_R} and Lemma \ref{lem:comp_sob_norm} that
\begin{align*}
    \|\varphi_{rec}\|_{W^{m, \infty}(\R)}
    &\leq \exp\{\cO(m\log(m + \|\tilde{\varphi}_{clip}\|_{W^{m, \infty}(\R)}) )\} \|\tilde{\varphi}_{rec}\|_{W^{m, \infty}([a/2, 2b])} \\
    &\leq \exp\{\cO(m^2\log(m/a) + m^2\log\log(1 / \eps_{clip}) + m N)\} .
\end{align*}
The choice of $\eps_{clip}$ given in \eqref{eq:rec_naive_eps_clip} suggests that
\begin{align*}
    \|\varphi_{rec}\|_{W^{m, \infty}(\R)} \leq \exp\{\cO(m^2\log(m N / a) + m^2\log\log(1 / \eps) + m N)\}.
\end{align*}
Lemma \ref{lem:clip_gelu_approx} together with \eqref{eq:rec_naive_eps_clip} imply that $\tilde{\varphi}_{clip} \in \NN(L_{clip}, W_{clip}, S_{clip}, B_{clip})$ with
\begin{align*}
    L_{clip} \vee \|W_{clip}\|_\infty \vee S_{clip} \lesssim 1,
    \quad \log B_{clip} \lesssim \log(1 / \eps) + mN + m\log(m / a) .
\end{align*}
Therefore, applying Lemma \ref{lem:concat_nn}, we obtain that $\varphi_{rec} \in \NN(L_{rec}, W_{rec}, S_{rec}, B_{rec})$ with the parameters specified in \eqref{eq:phi_rec_naive_cfg_aux}.
The proof is complete.

\myendproof

\subsection{Proof of Lemma \ref{lem:div_gelu_approx}}
\label{sec:lem_div_gelu_approx_proof}

\noindent
\textbf{Step 1: approximation error decomposition.}\quad
Let $\varphi_{rec}$ be a reciprocal function approximation from Lemma \ref{lem:recip_gelu_approx} with the accuracy parameter $\eps_0 \in (0, 1)$ and $\varphi_{id}$ be the identity approximation from Lemma \ref{lem:id_deep_gelu_approx} with the accuracy parameter $\eps_0$, the scale parameter $1$, and the number of layers of $\varphi_{rec}$.
Let also $\varphi_{mul}$ be a multiplication network from Corollary \ref{co:multi_approx_gelu} with the precision parameter $\eps_0$.
We put the smoothness parameter $m + 1$ for all the networks.
We also refer to $f_{rec}(x) = 1 / x$ for any $x > 0$ as a reciprocal function.
Therefore, for $\varphi_{div} = \varphi_{mul}(\varphi_{id}, \varphi_{rec})$ we have due to the triangle inequality that
\begin{align}
    \label{eq:phi_div_acc_aux}
    &\| \varphi_{mul}(\varphi_{id}, \varphi_{rec}) - \id \cdot f_{rec} \|_{W^{m, \infty}([-1, 1] \times [a_0, 1])} \\
    \notag
    &\quad \leq \underbrace{ \| \varphi_{mul}(\varphi_{id}, \varphi_{rec}) - \varphi_{mul}(\id, f_{rec}) \|_{W^{m, \infty}([0, 1] \times [a_0, 1])} }_{(A)}
    + \underbrace{ \| \varphi_{mul}(\id, f_{rec}) - \id \cdot f_{rec} \|_{W^{m, \infty}([-1, 1] \times [a_0, 1])} }_{(B)} .
\end{align}
Next, we evaluate the terms $(A)$ and $(B)$ individually.

\noindent
\textbf{Step 2: bounding term $(A)$.}\quad
Lemma \ref{lem:comp_sob_norm} together with Corollary \ref{co:multi_approx_gelu} suggests that
\begin{align*}
    &\| \varphi_{mul}(\varphi_{id}, \varphi_{rec}) - \varphi_{mul}(\id, f_{rec}) \|_{W^{m, \infty}([-1, 1] \times [a_0, 1])} \\
    &\quad \leq \exp\{\cO(m\log(m + \|\id\|_{W^{m, \infty}([-1, 1])} + \|f_{rec}\|_{W^{m, \infty}([a_0, 1])} ))\}\eps_0.
\end{align*}
Note that Stirling's approximation yields
\begin{align}
    \label{eq:f_rec_a_0_one}
    \|f_{rec}\|_{W^{m, \infty}([a_0, 1])} \leq a_0^{-(m + 1)}m!
    = \exp\{\cO(mN + m\log m)\}.
\end{align}
This observation implies that
\begin{align}
    \label{eq:phi_div_A_bound}
    \| \varphi_{mul}(\varphi_{id}, \varphi_{rec}) - \varphi_{mul}(\id, f_{rec}) \|_{W^{m, \infty}([-1, 1] \times [a_0, 1])}
    \leq \exp\{\cO(m^2N + m^2\log m)\}\eps_0.
\end{align}

\noindent
\textbf{Step 3: bounding term $(B)$.}\quad
The bound is obtained in a similar way.
Formally, Lemma \ref{lem:comp_sob_norm} implies that
\begin{align*}
    &\| \varphi_{mul}(\id, f_{rec}) - \id \cdot f_{rec} \|_{W^{m, \infty}([-1, 1] \times [a_0, 1])} \\
    &\quad \leq \exp\{\cO(m\log(m + \|\id\|_{W^{m, \infty}([-1, 1])} + \|f_{rec}\|_{W^{m, \infty}([a_0, 1])} ))\}\eps_0.
\end{align*}
Using \eqref{eq:f_rec_a_0_one}, we arrive at
\begin{align}
    \label{eq:phi_div_B_bound}
    \| \varphi_{mul}(\id, f_{rec}) - \id \cdot f_{rec} \|_{W^{m, \infty}([-1, 1] \times [a_0, 1])}
    \leq \exp\{\cO(m^2N + m^2\log m)\}\eps_0 .
\end{align}

\noindent
\textbf{Step 4: combining $(A)$ and $(B)$ together.}\quad
From \eqref{eq:phi_div_acc_aux}, \eqref{eq:phi_div_A_bound} and \eqref{eq:phi_div_B_bound}
we deduce that setting
\begin{align}
    \label{eq:phi_div_eps_0_def}
    \log(1 / \eps_0) \asymp m^2 N + m^2 \log m + \log(1 / \eps),
\end{align}
ensures that
\begin{align*}
    \| \varphi_{mul}(\varphi_{id}, \varphi_{rec}) - \mathrm{div} \|_{W^{m, \infty}([-1, 1] \times [a_0, 1])}
    \leq \eps .
\end{align*}
Moreover, using Corollary \ref{co:multi_approx_gelu} together with Lemmata \ref{lem:id_deep_gelu_approx}, \ref{lem:recip_gelu_approx}, and \ref{lem:comp_sob_norm}, we deduce that
\begin{align*}
    \|\varphi_{div}\|_{W^{m, \infty}(\R^2)}
    &\leq \exp\{\cO(m\log(m + \|\varphi_{id}\|_{W^{m, \infty}(\R)} + \|\varphi_{rec}\|_{W^{m, \infty}(\R)}))\} \\
    &\leq \exp\{\cO(m^4 N + m^4\log(m\log(1 / \eps)))\} .
\end{align*}

\noindent
\textbf{Step 5: deriving the configuration of $\varphi_{div}$.}\quad
We find from \eqref{eq:phi_div_eps_0_def} and Lemma \ref{lem:recip_gelu_approx} that $\varphi_{rec}$ belongs to the neural network class $\NN(L_{rec}, W_{rec}, S_{rec}, B_{rec})$ with
\begin{align*}
    &L_{rec} \lesssim \log(mN\log(1 / \eps)),
    \quad \|W_{rec}\|_\infty \vee S_{rec} \lesssim m^8 N (N^4 + m^4\log^4(1 / \eps_0))
    \lesssim m^{21} N^5 \log^4(1 / \eps), \\
    &\log B_{rec} \lesssim m^{24} N^4 \log^4(1 / \eps) .
\end{align*}
Moreover, the bound for $L_{rec}$ together with Lemma \ref{lem:id_deep_gelu_approx} suggest that $\varphi_{id} \in \NN(L_{id}, W_{id}, S_{id}, B_{id})$ with
\begin{align*}
    &L_{id} \vee S_{id} \lesssim L_{rec} \lesssim \log(mN\log(1 / \eps)),
    \quad \|W_{id}\|_\infty \lesssim 1, \\
    &\log B_{id} \lesssim (m + L_{rec})\log m + \log(1 / \eps_0)
    \lesssim m^2(N + \log m) + \log m \cdot \log(1 / \eps) .
\end{align*}
In addition, due to Corollary \ref{co:multi_approx_gelu}, it holds that $\varphi_{mul} \in \NN(L_{mul}, W_{mul}, S_{mul}, B_{mul})$ with
\begin{align*}
    L_{mul} \vee \|W_{mul}\| \vee S_{mul} \lesssim 1,
    \quad \log B_{mul} \lesssim m^2 N + m^2 \log m + \log(1 / \eps).
\end{align*}
Therefore, Lemma \ref{lem:paral_nn} and Lemma \ref{lem:concat_nn} imply that $\varphi_{div}$ has
\begin{align*}
    L \lesssim \log(mN\log(1 / \eps)),
    \quad \|W\|_\infty \vee S \lesssim m^{21} N^5 \log^4(1 / \eps),
    \quad \log B \lesssim m^{24} N^4 \log^4(1 / \eps).
\end{align*}
The proof is complete.

\myendproof

\section{Auxiliary results}

\begin{Lem}[evaluation of Hermite polynomials, \cite{puchkin2024breaking}, Appendix D]
\label{lem:herm_poly_bound}
    For any $n \in \N$ we define a "probabilist's" Hermite polynomial
    \begin{align*}
        \cH_n(x) = (-1)^n e^{x^2 / 2} \frac{\dd^n}{\dd x^n}e^{-x^2 / 2}, \quad x \in \R.
    \end{align*}
    Then it holds that
    \begin{align*}
        \max_{x \in \R}\left|\cH_n(x) e^{-x^2 / 4}\right| \leq \sqrt{n!} \quad \text{for all } n \in \N .
    \end{align*}
\end{Lem}

\begin{Lem}[properties of GELU acitvation function]
\label{lem:gelu_seminorms_bound}
    For any $k \in \N$ we have the following bounds for the Sobolev seminorms:
    \begin{align*}
        \left|\gelu\right|_{W^{k, \infty}(\R)} \leq
        \begin{cases}
            1 + 1 / \sqrt{2\pi}, \quad &k = 1, \\
            (k + 1)\sqrt{\frac{(k - 2)!}{2\pi}}, \quad &k \geq 2
        \end{cases}
    \end{align*}
    For $k = 0$ we have that
    \begin{align*}
        \|\gelu\|_{W^{0, \infty}([-C, C])} \leq C, \quad \text{for all } C > 0.
    \end{align*}
    In addition, for any $A \geq 0$ and $m \in \N$, the tails behave as follows:
    \begin{align*}
        \|\gelu - \id\|_{W^{m, \infty}([A, +\infty))} \vee \|\gelu\|_{W^{m, \infty}((-\infty, -A])}
        \leq 2 e^{-A^2 / 4} \sqrt{m!} .
    \end{align*}
\end{Lem}

\begin{proof}
    We first recall that
    \begin{align*}
        \gelu(x) = x \cdot \Phi(x), \quad \Phi(x) = \frac{1}{\sqrt{2\pi}}\integral{-\infty}^{x} e^{-t^2 / 2} \, \dd t,
    \end{align*}
    which immediately implies that $\|\gelu\|_{W^{0, \infty}([-C, C])} \leq C$ for any $C > 0$ and 
    \begin{align}
    \label{eq:gelu_deriv}
        \partial^1\gelu(x) = \Phi(x) - \frac{1}{\sqrt{2\pi}}\frac{\dd}{\dd x}e^{-x^2 / 2}.
    \end{align}
    Hence, using Lemma \ref{lem:herm_poly_bound} together with the observation that $\partial^k e^{-x^2/2} = (-1)^k e^{-x^2 / 2} \, \cH_k(x)$ for any $k \in \N$, where
    \begin{align*}
        \cH_k(x) = (-1)^k e^{x^2 / 2} \, \partial^k e^{-x^2 / 2}, \quad x \in \R ,
    \end{align*}
    we obtain that $|\gelu|_{W^{1, \infty}(\R)} \leq 1 + 1 / \sqrt{2\pi}$.
    Subsequently, from \eqref{eq:gelu_deriv} we deduce that for any $k \in \N$ with $k \geq 2$
    \begin{align}
        \label{eq:gelu_k_der}
        \partial^k\gelu(x) = \frac{1}{\sqrt{2\pi}}\left(\partial^{k - 2}(e^{-x^2 / 2}) - \partial^k(e^{-x^2 / 2}) \right).
    \end{align}
    Applying Lemma \ref{lem:herm_poly_bound}, we have that
    \begin{align*}
        |\gelu|_{W^{k, \infty}(\R)}
        \leq \frac{1}{\sqrt{2\pi}}\left(\sqrt{(k - 2)!} + \sqrt{k!}\right)
        \leq (k + 1)\sqrt{ \frac{(k - 2)!}{2\pi} },
    \end{align*}
    which validates the first claim of the statement.
    Now focus on the behavior of tails.
    First, consider
    \begin{align*}
        \|\gelu - \id\|_{W^{0, \infty}([A, +\infty))}
        = \sup_{x \geq A} x (1 - \Phi(x))
        \leq \sup_{x \geq A} x e^{-x^2 / 2}
        \leq e^{-A^2 / 4}\sqrt{2}e^{-1/2},
    \end{align*}
     where the penultimate inequality uses Gaussian tails and the last inequality follows from the observation that $x e^{-x^2/4} \leq \sqrt{2}e^{-1/2}$ for all $x \in \R$.
     Similarly,
     \begin{align}
         \|\gelu\|_{W^{0, \infty}((-\infty, A])}
         = \sup_{x \leq -A} |x \Phi(x)| = \sup_{x \geq A} x (1 - \Phi(x))
         \leq e^{-A^2 / 4}\sqrt{2}e^{-1/2}.
     \end{align}
     As for the derivatives, we have
     \begin{align*}
         | \gelu - \id |_{W^{1, \infty}([A, +\infty))}
         &\leq \sup_{x \geq A}(1 - \Phi(x)) + (\sqrt{2\pi})^{-1}\sup_{x \geq A} xe^{-x^2 / 2} \\
         &\leq e^{-A^2 / 2} + (\sqrt{2\pi})^{-1} \sqrt{2}e^{-1/2}e^{-A^2 / 4} \\
         &\leq 2 e^{-A^2 / 4}
     \end{align*}
     and also
     \begin{align*}
         |\gelu|_{W^{1, \infty}((-\infty, -A])}
         &\leq \sup_{x \leq -A} \Phi(x) + (\sqrt{2\pi})^{-1}\sup_{x \leq -A} |xe^{-x^2 / 2}| \\
         &\leq e^{-A^2 / 2} + (\sqrt{2\pi})^{-1} \sqrt{2}e^{-1/2}e^{-A^2 / 4} \\
         &\leq 2 e^{-A^2 / 4}.
     \end{align*}
     Now for any natural $k \geq 2$ we have from \eqref{eq:gelu_k_der} that
     \begin{align*}
         |\gelu|_{W^{k, \infty}((-\infty, -A] \cup [A, +\infty) )}
         \leq (\sqrt{2\pi})^{-1}\left( \sup_{|x| \geq A}|e^{-x^2/2} \cH_k(x)| + \sup_{|x| \geq A}|e^{-x^2/2} \cH_{k-2}(x)| \right).
     \end{align*}
     Hence, Lemma \ref{lem:herm_poly_bound} implies that
    \begin{align*}
        |\gelu|_{W^{k, \infty}((-\infty, -A] \cup [A, +\infty) )}
        \leq \sqrt{\frac{2}{\pi}} e^{-A^2 / 4} \sqrt{k!}.
    \end{align*}
    Combining all together, we have that
    \begin{align*}
        \|\gelu - \id\|_{W^{m, \infty}([A, +\infty))} \vee \|\gelu\|_{W^{m, \infty}((-\infty, -A])}
        \leq 2 e^{-A^2 / 4} \sqrt{m!}.
    \end{align*}
    The proof is now complete.
     
\end{proof}

\begin{Lem}[\cite{de2021approximation}, Lemma A.6]
    \label{lem:prod_sob_norm}
    Let $d \in \N$, $k \in \Z_+$, $\Omega \subseteq \R^d$ and $f, g \in W^{k, \infty}(\Omega)$.
    Then it holds that
    \begin{align*}
        \|f \cdot g\|_{W^{k, \infty}(\Omega)} \leq 2^k \|f\|_{W^{k, \infty}(\Omega)} \|g\|_{W^{k, \infty}(\Omega)}.
    \end{align*}
\end{Lem}

\begin{Lem}[\cite{de2021approximation}, Lemma A.7]
    \label{lem:comp_sob_norm}
    Let $d, m, n \in \N$ and let also $\Omega_1 \subseteq \R^d$, $\Omega_2 \subseteq \R^m$, $f \in C^n(\Omega_1, \Omega_2)$ and $g \in C^n(\Omega_2, \R)$. Then it holds that
    \begin{align*}
        \|g \circ f\|_{W^{n, \infty}(\Omega_1)} \leq 16 (e^2n^4md^2)^n \|g\|_{W^{n, \infty}(\Omega_2)} \max_{1 \leq i \leq m}(\|(f)_i\|_{W^{n, \infty}(\Omega_1)}^n \vee 1).
    \end{align*}
    Moreover, if $g \in C^{n + 1}(\Omega_2, \R)$ and $\tilde f \in C^n(\Omega_1, \Omega_2)$, then
    \begin{align*}
        &\|g \circ f - g \circ \tilde{f}\|_{W^{n, \infty}(\Omega_1)} \\
        &\quad \leq 32(e^2n^5 m^2 d^2)^n \|g\|_{W^{n + 1, \infty}(\Omega_2)} \max_{1 \leq i \leq m}\|(f)_i - (\tilde{f})_i\|_{W^{n, \infty}(\Omega_1)} \left(1 \vee \|(f)_i\|^{2n}_{W^{n, \infty}(\Omega_1)} \vee \|(\tilde{f})_i\|^{2n}_{W^{n, \infty}(\Omega_1)}\right).
    \end{align*}
\end{Lem}

\begin{proof}
    We reprove Lemma A.7 from \cite{de2021approximation}, correcting a minor technical oversight in the original derivation.
    Specifically, their bound omits $\max(1, \cdot)$ term, which we include here for correctness.
    We begin with the multivariate Faà di Bruno formula \cite[Theorem 2.1]{constantine_faa}.
    For $\bnu \in \Z_+^d$ with $|\bnu| = q$ for some $q \in \N$ with $q \leq n$ it holds that
    \begin{align}
        \label{eq:faa_di_bruno}
        \partial^\bnu(g \circ f) = \sum_{1 \leq |\blambda| \leq q} \partial^\blambda g \sum_{p(\bnu, \blambda)} (\bnu !)\prod_{j=1}^q \frac{(f_{l_j})^{k_j}}{k_j! (l_j!)^{k_j}},
    \end{align}
    where $(f_\mu)_i = \partial^{\boldsymbol\mu} f_i$ for $1 \leq i \leq m$.
    In addition,
    \begin{align*}
        p(\bnu, \blambda) = \big\{ &(\bk_1, \dots, \bk_q; \bl_1, \dots, \bl_q) : \text{for some } 1 \leq s \leq q, \\
        &\bk_i = 0_m \text{ and } \bl_i = 0_d \text{ for all } 1 \leq i \leq q - s; \; |\bk_i| > 0 \text{ for all } q - s + 1 \leq i \leq n; \\
        & \text{and } 0_d \prec \bl_{q - s + 1} \prec \dots \prec \bl_q \text{ are such that} \\
        &\sum_{i=1}^n \bk_i = \blambda, \; \sum_{i=1}^n |\bk_i|\bl_i = \bnu \big\},
    \end{align*}
    where we write $\boldsymbol{a} \prec \boldsymbol{b}$ if either $|\boldsymbol{a}| \leq |\boldsymbol{b}|$ or $|\boldsymbol{a}| = |\boldsymbol{b}|$ and $\boldsymbol{a}_1 < \boldsymbol{b}_1$ or $|\boldsymbol{a}| = |\boldsymbol{b}|$ and for some $1 \leq k \leq d - 1$ we have $\boldsymbol{a}_{k + 1} < \boldsymbol{b}_{k + 1}$ with $\boldsymbol{a}_1 = \boldsymbol{b}_1, \dots, \boldsymbol{a}_k = \boldsymbol{b}_k$.
    It is evident that in \eqref{eq:faa_di_bruno} we have $\sum_{i=1}^n |\bk_i| \leq n$ and, hence, the number of $(\bk_1, \dots, \bk_n)$ satisfying the definition of $p(\bnu, \blambda)$ is bounded by $|P_{n, \, nm + 1}|$, which is then evaluated as $\sqrt{\pi}e^n(mn)^n$ according to Lemma 2.1 from \cite{de2021approximation}.
    Similarly, the number of $(\bl_1, \dots, \bl_n)$ satisfying the definition of $p(\bnu, \blambda)$ is bounded by $|P_{n, dn + 1}|$, which in turn, is bounded by $\sqrt{\pi} e^n (dn)^n$.
    This results in
    \begin{align}
        \label{eq:p_nu_lambda_bound}
        |p(\bnu, \blambda)| \leq \pi(e^2 n^2 md )^n.
    \end{align}
    Finally, evaluate
    \begin{align}
        \label{eq:lambda_p_n_d_g_bounds}
        |\{ \blambda \in \Z_+^d \; : \; 1 \leq |\blambda| \leq q \}|
        \leq |P_{n, \, d + 1}|
        \leq \sqrt{\pi} e^n d^n, \quad |\partial^\blambda g| \leq \|g\|_{W^{n, \infty}(\Omega_2)},
        \quad \bnu ! \leq n!
    \end{align}
    and 
    \begin{align}
        \label{eq:prod_f_bound}
        \prod_{j=1}^n (f_{l_j})^{k_j}
        \leq 1 \vee \max_{1\leq i \leq m}\|(f)_i\|^n_{W^{n, \infty}(\Omega_1)},
    \end{align}
    Therefore, Stirling's approximation implies that
    \begin{align*}
        \max_{\bnu \in \Z_+^d, \; 1 \leq |\bnu| \leq n}\|\partial^\bnu g \circ f\|_{W^{0, \infty}(\Omega_1)}
        &\leq \sqrt{\pi}(ed)^n \|g\|_{W^{n, \infty}(\Omega_2)} \pi (e^2n^2md)^n n! (1 \vee \max_{1\leq i \leq m}\|(f)_i\|^n_{W^{n, \infty}(\Omega_1)}) \\
        &\leq 16(e^2 n^4 md^2)^n \|g\|_{W^{n, \infty}(\Omega_2)} (1 \vee \max_{1\leq i \leq m}\|(f)_i\|^n_{W^{n, \infty}(\Omega_1)}).
    \end{align*}
    For $\bnu = 0_d$ we have that
    \begin{align*}
        \|g \circ f\|_{W^{0, \infty}(\Omega_1)} \leq \|g\|_{W^{m, \infty}(\Omega_2)}
        \leq 16(e^2 n^4 md^2)^n \|g\|_{W^{n, \infty}(\Omega_2)} (1 \vee \max_{1\leq i \leq m}\|(f)_i\|^n_{W^{n, \infty}(\Omega_1)}).
    \end{align*}
    Hence, the first claim holds true.
    Now using \eqref{eq:faa_di_bruno}, we deduce that
    \begin{align*}
        &|\partial^\bnu (g \circ f) - \partial^\bnu (g \circ \tilde{f})| \\
        &\quad \leq \sum_{1 \leq |\blambda| \leq q} |\partial^\blambda [g] \circ f - \partial^\blambda [g] \circ \tilde{f}|\sum_{p(\bnu, \blambda)}(\bnu !)\prod_{j=1}^q \frac{(f_{l_j})^{k_j}}{k_j! (l_j!)^{k_j}} \\
        &\quad + \sum_{1 \leq |\blambda| \leq q}|\partial^\blambda [g] \circ \tilde{f}|\sum_{p(\bnu, \blambda)}(\bnu!)\frac{|\prod_{j=1}^q(f_{l_j})^{k_j} - \prod_{j=1}^q (\tilde{f}_{l_j})^{k_j}|}{\prod_{j=1}^q k_j! (l_j!)^{k_j}}.
    \end{align*}
    First, bound the first term.
    From \eqref{eq:lambda_p_n_d_g_bounds} and \eqref{eq:prod_f_bound} we find that
    \begin{align*}
        &\sum_{1 \leq |\blambda| \leq q} |\partial^\blambda [g] \circ f - \partial^\blambda [g] \circ \tilde{f}|\sum_{p(\bnu, \blambda)}(\bnu !)\prod_{j=1}^q \frac{(f_{l_j})^{k_j}}{k_j! (l_j!)^{k_j}} \\
        &\quad\leq |\{\blambda \in \Z_+^d \; : \; 1 \leq |\blambda| \leq q \}| \cdot |p(\bnu, \blambda)|\cdot n! \cdot (1 \vee \max_{1\leq i \leq m}\|(f)_i\|^n_{W^{n, \infty}(\Omega_1)}) |\partial^\blambda [g] \circ f - \partial^\blambda [g] \circ \tilde{f}|.
    \end{align*}
    Now mean value theorem together with \eqref{eq:p_nu_lambda_bound} and \eqref{eq:lambda_p_n_d_g_bounds} suggests that
    \begin{align}
        \label{eq:faa_first_t}
        \notag
        &\sum_{1 \leq |\blambda| \leq q} |\partial^\blambda [g] \circ f - \partial^\blambda [g] \circ \tilde{f}|\sum_{p(\bnu, \blambda)}(\bnu !)\prod_{j=1}^q \frac{(f_{l_j})^{k_j}}{k_j! (l_j!)^{k_j}} \\
        &\quad \leq 16 m (e^2 n^4 md^2)^n (1 \vee \max_{1\leq i \leq m}\|(f)_i\|^n_{W^{n, \infty}(\Omega_1)}) \|g\|_{W^{n + 1, \infty}(\Omega_2)}\max_{1 \leq i \leq m}\|(f)_i - (\tilde{f})_i\|_{W^{0, \infty}(\Omega_1)}.
    \end{align}
    Second, evaluate the second term, using \eqref{eq:p_nu_lambda_bound} and \eqref{eq:lambda_p_n_d_g_bounds}:
    \begin{align*}
        &\sum_{1 \leq |\blambda| \leq q}|\partial^\blambda [g] \circ \tilde{f}|\sum_{p(\bnu, \blambda)}(\bnu!)\frac{|\prod_{j=1}^q(f_{l_j})^{k_j} - \prod_{j=1}^q (\tilde{f}_{l_j})^{k_j}|}{\prod_{j=1}^q k_j! (l_j!)^{k_j}} \\
        &\quad \leq \sqrt{\pi} e^n d^n \cdot \pi(e^2 n^2 md )^n \cdot \|g\|_{W^{n, \infty}(\Omega_2)} n! \cdot \sum_{j=1}^q |(f_{l_j})^{k_j} - (\tilde{f}_{l_j})^{k_j}| \prod_{u < j} |(f_{l_u})^{k_u}|
        \prod_{u > j} |(f_{l_u})^{k_u}|.
    \end{align*}
    Therefore, Stirling's approximation suggests that
    \begin{align*}
        &\sum_{1 \leq |\blambda| \leq q}|\partial^\blambda [g] \circ \tilde{f}|\sum_{p(\bnu, \blambda)}(\bnu!)\frac{|\prod_{j=1}^q(f_{l_j})^{k_j} - \prod_{j=1}^q (\tilde{f}_{l_j})^{k_j}|}{\prod_{j=1}^q k_j! (l_j!)^{k_j}} \\
        &\quad \leq 16(e^2 n^4 md^2)^n \|g\|_{W^{n, \infty}(\Omega_2)}  \sum_{j=1}^q |(f_{l_j})^{k_j} - (\tilde{f}_{l_j})^{k_j}| \max_{1 \leq i \leq m}(1 \vee \|(f)_i\|^n_{W^{n, \infty}(\Omega_1)} \vee \|(\tilde{f})_i\|^n_{W^{n, \infty}(\Omega_1)}).
    \end{align*}
    We next note that
    \begin{align*}
        \sum_{j=1}^q |(f_{l_j})^{k_j} - (\tilde{f}_{l_j})^{k_j}|
        &= \sum_{j=1}^q \sum_{i=1}^m |(f_{l_j})_i^{(k_j)_i} - (\tilde{f}_{l_j})_i^{(k_j)_i}| \prod_{u < i} |(f_{l_j})_u^{(k_j)_u}| \prod_{u > i} |(\tilde{f}_{l_j})_u^{(k_j)_u}| \\
        &\leq \sum_{j=1}^q |k_j| \cdot \max_{1 \leq i \leq m}\|(f)_i - (\tilde{f})_i\|_{W^{n, \infty}(\Omega_1)} (1 \vee \|(f)_i\|^n_{W^{n, \infty}(\Omega_1)} \vee \|(\tilde{f})_i\|^n_{W^{n, \infty}(\Omega_1)}).
    \end{align*}
    Therefore, we obtain that
    \begin{align}
        \label{eq:faa_sec_term}
        \notag
        &\sum_{1 \leq |\blambda| \leq q}|\partial^\blambda [g] \circ \tilde{f}|\sum_{p(\bnu, \blambda)}(\bnu!)\frac{|\prod_{j=1}^q(f_{l_j})^{k_j} - \prod_{j=1}^q (\tilde{f}_{l_j})^{k_j}|}{\prod_{j=1}^q k_j! (l_j!)^{k_j}} \\
        &\quad \leq 16 n (e^2 n^4 md^2)^n \|g\|_{W^{n, \infty}(\Omega_2)} \max_{1 \leq i \leq m}\|(f)_i - (\tilde{f})_i\|_{W^{n, \infty}(\Omega_1)} (1 \vee \|f_i\|^{2n}_{W^{n, \infty}(\Omega_1)} \vee \|(\tilde{f})_i\|^{2n}_{W^{n, \infty}(\Omega_1)}).
    \end{align}
    For $\bnu = 0_d$ it holds that
    \begin{align*}
        \|g \circ f - g \circ \tilde{f}\|_{W^{0, \infty}(\Omega_1)}
        \leq m \|g\|_{W^{1, \infty}(\Omega_2)} \max_{1 \leq i \leq m}\|(f)_i - (\tilde{f})_i\|_{W^{n, \infty}(\Omega_1)}.
    \end{align*}
    Thus, from \eqref{eq:faa_first_t} and \eqref{eq:faa_sec_term} we conclude that
    \begin{align*}
        &\|g \circ f - g \circ \tilde{f}\|_{W^{n, \infty}(\Omega_1)} \\
        &\quad \leq 32(e^2n^5 m^2 d^2)^n \|g\|_{W^{n + 1, \infty}(\Omega_2)} \max_{1 \leq i \leq m}\|(f)_i - (\tilde{f})_i\|_{W^{n, \infty}(\Omega_1)} (1 \vee \|(f)_i\|^{2n}_{W^{n, \infty}(\Omega_1)} \vee \|(\tilde{f})_i\|^{2n}_{W^{n, \infty}(\Omega_1)}).
    \end{align*}
    The proof is complete.

\end{proof}

\begin{Lem}[concatenation of neural networks]
    \label{lem:concat_nn}
    Given $K \in \N$ with $K \geq 2$. Then, for any neural networks $\varphi^{(k)} \in \NN(L_k, W_k, S_k, B_k)$ with $1 \leq k \leq K$
    such that $\varphi^{(k)} : \R^{d_k} \to \R^{d_{k + 1}}$, there exists a neural network $h = \varphi^{(K)} \circ \varphi^{(K - 1)} \dots \circ \varphi^{(1)} \in \NN(L, W, S, B)$ satisfying
    \begin{align*}
        L &\leq 1 + \sum_{k=1}^K (L_k - 1), \quad S \leq \sum_{k=1}^K S_k + 2\sum_{k = 1}^{K - 1}\|W_k\|_\infty \cdot \|W_{k + 1}\|_\infty, \\
        \|W\|_{\infty} &\leq \max_{1 \leq k \leq K}\|W_k\|_{\infty},
        \quad B \leq 2 \max_{1 \leq k \leq K - 1} \left[ (B_k \vee 1) (B_{k + 1} \vee 1) \left( \|W_k\|_\infty \vee \|W_{k + 1}\|_\infty \right) \right].
    \end{align*}
        
\end{Lem}

\begin{proof}
    It suffices to prove the statement for $K = 2$, since one can easily generalize it to $K \geq 3$ by induction.
    Recall that each $\varphi^{(j)}$ for $j \in \{1, 2\}$ admits the representation given in \eqref{eq:feed_forward_nn_def}.
    Specifically,
    \begin{align*}
        \varphi^{(j)}(x) = -b^j_{L_j} + A^j_{L_j} \circ \gelu_{b^j_{L_j - 1}} \circ A^j_{L_j - 1} \circ \gelu_{b^j_{L_j - 2}} \circ \dots \circ A^j_2 \circ \gelu_{b^j_1} \circ A^j_1 \circ x .
    \end{align*}
    Therefore, we deduce that
    \begin{align*}
        \varphi^{(2)} \circ \varphi^{(1)} \circ x = -b^{2}_{L_2} + A^{2}_{L_2} \circ \gelu_{b^{2}_{L_2 - 1}} \circ \dots \circ \gelu_{A_1^{2} b^{1}_{L_1} + b_1^2}
        \circ A_1^{2}A_{L_1}^{1} \circ \dots \circ \gelu_{b_1^{1}} \circ A_1^{1} \circ x .
    \end{align*}
    Consequently, it follows that $\varphi^{(2)} \circ \varphi^{(1)} \in \NN(L, W, S, B)$ with
    \begin{align*}
        L &\leq L_1 + L_2 - 1,
        \quad \|W\|_\infty \leq \|W_1\|_\infty \vee \|W_2\|_\infty, \\
        S &\leq S_1 + S_2 + 2 \|W_1\|_\infty \cdot \|W_2\|_\infty,
        \quad B \leq 2 (B_1 \vee 1) (B_2 \vee 1) (\|W_1\|_\infty \vee \|W_2\|_\infty) .
    \end{align*}
    Hence, the base case holds.
    The result then follows by induction.

\end{proof}

\begin{Lem}[parallelization of neural networks]
    \label{lem:paral_nn}
    Let $K \in \N$ with $K \geq 2$ and let neural networks $\varphi^{(k)} \in \NN(L_k, W_k, S_k, B_k)$ for $1 \leq k \leq K$.
    Assume further that $L_k = L$ for all $1 \leq k \leq K$.
    Then, the following holds:
    \begin{itemize}
        \item[(i)] if $\varphi^{(k)} : \R^{d_k} \to \R$ for each $1 \leq k \leq K$, then there exists a neural network $\varphi \in \NN(L, W, S, B)$ such that $\varphi(x) = (\varphi^{(1)}(x_1), \dots, \varphi^{(K)}(x_K))^\top$ for all $x = (x_1^\top, \dots, x_K^\top)^\top$,
        where $x_k \in \R^{d_k}$ for any $1 \leq k \leq K$.
        In addition, there exists $\varphi_{sum} \in \NN(L, W, S, B_{sum})$, which implements the summation, that is,
        \begin{align*}
            \varphi_{sum}(x) = \sum_{k=1}^K \varphi^{(k)}(x_k), \quad \text{for all } x = (x_1^\top, \dots, x_K^\top)^\top .
        \end{align*}
        \item[(ii)] if $\varphi^{(k)} : \R^p \to \R$ for some $p \in N$ for every $1 \leq k \leq K$, then there exists a neural network $\varphi \in \NN(L, W, S, B)$ satisfying $\varphi(x) = (\varphi^{(1)}(x), \dots, \varphi^{(K)}(x))^\top$ for all $x \in \R^p$.
        Moreover, there exists a summation network $\varphi_{sum} \in \NN(L, W, S, B_{sum})$ such that
        \begin{align*}
            \varphi_{sum}(x) = \sum_{k=1}^K \varphi^{(k)}(x), \quad \text{for all } x \in \R^p .
        \end{align*}
    \end{itemize}
    Furthermore, in both cases it holds that
    \begin{align*}
        \|W\|_\infty \leq \sum_{k=1}^K \|W^{(k)}\|_\infty,
        \quad S \leq \sum_{k = 1}^K S^{(k)},
        \quad B \leq \max_{1 \leq k \leq K} B^{(k)},
        \quad B_{sum} \leq K \max_{1 \leq k \leq K} B^{(k)} .
    \end{align*}

\end{Lem}

\begin{proof}
    As for the case $(i)$, from \eqref{eq:feed_forward_nn_def} we find that $\varphi^{(j)}$ for each $1 \leq j \leq K$ has the following form:
    \begin{align*}
        \varphi^{(j)}(x_j) = -b^j_{L} + A^j_{L} \circ \gelu_{b^j_{L - 1}} \circ A^j_{L - 1} \circ \gelu_{b^j_{L - 2}} \circ \dots \circ A^j_2 \circ \gelu_{b^j_1} \circ A^j_1 \circ x_j .
    \end{align*}
    Following \cite{nakada2020adaptive}, we introduce
    \begin{align}
        \label{eq:tilde_A_paral_nn}
        \tilde{A}_l =
        \begin{pmatrix}
            A_l^1 & 0 & \dots & 0 \\
            0 & A_l^2 & \dots & 0 \\
            \vdots & \vdots & \ddots & \vdots \\
            0 & 0 & \dots & A_l^K
        \end{pmatrix},
        \quad \tilde{b}_l =
        \begin{pmatrix}
            b_l^1 \\ \vdots \\ b_l^K    
        \end{pmatrix},
        \quad \text{for all } 1 \leq l \leq L .
    \end{align}
    Hence, as suggested by \eqref{eq:feed_forward_nn_def}, for
    \begin{align*}
        \varphi(x) = -\tilde{b}_L + \tilde{A}_L \circ \gelu_{\tilde{b}_{L - 1}} \circ \tilde{A}_{L - 1} \circ \gelu_{\tilde{b}_{L - 2}} \circ \dots \circ \tilde{A}_2 \circ \gelu_{\tilde{b}_1} \circ \tilde{A}_1 \circ x 
    \end{align*}
    we have that $\varphi(x) = (\varphi^{(1)}(x_1), \dots, \varphi^{(K)}(x_K))^\top$ for all $x = (x_1^\top, \dots, x_K^\top)^\top$ .
    Furthermore, the configuration of the network $\varphi$ coincides with that from the statement of the lemma.
    As for the summation network, we let
    \begin{align*}
        \varphi_{sum}(x) = -\bar{b}_L + \bar{A}_L \circ \gelu_{\tilde{b}_{L - 1}} \circ \tilde{A}_{L - 1} \circ \gelu_{\tilde{b}_{L - 2}} \circ \dots \circ \tilde{A}_2 \circ \gelu_{\tilde{b}_1} \circ \tilde{A}_1 \circ x ,
    \end{align*}
    where
    \begin{align}
        \label{eq:bar_A_paral_nn}
        \bar{A}_L =
        \begin{pmatrix}
            A_L^1 & A_L^2 & \dots & A_L^K
        \end{pmatrix},
        \quad \bar{b}_L = \sum_{k = 1}^K b_L^k .
    \end{align}
    Hence, it follows that
    \begin{align*}
        \varphi_{sum}(x) = \sum_{k = 1}^K \varphi^{(k)}(x_k), \quad \text{for all } x = (x_1^\top, \dots, x_K^\top)^\top .
    \end{align*}
    The configuration of $\varphi_{sum}$ immediately follows from \eqref{eq:bar_A_paral_nn}.
    The proof of the case $(ii)$ is identical to the considered one.
    The only difference is that in \eqref{eq:tilde_A_paral_nn} for $l = 1$ we define
    \begin{align*}
        \tilde{A}_1 = \begin{pmatrix}{A^1_1}^\top & {A_1^2}^\top & \dots & {A_1^K}^\top \end{pmatrix}^\top,
        \quad \tilde{b}_1 = \begin{pmatrix} {b_1^1}^\top & {b_1^2}^\top & \dots & {b_1^K}^\top \end{pmatrix}^\top .
    \end{align*}
    The proof is finished.

\end{proof}

\end{document}